\documentclass[10pt]{article}

\usepackage[usenames,dvipsnames,svgnames,table]{xcolor}

\usepackage{authblk}

\usepackage{url}

\usepackage{amssymb,amsmath}
\usepackage{amsthm}

\usepackage{geometry,placeins}
\usepackage{algorithmic}
\usepackage{algorithm}
\usepackage{graphicx}
\usepackage{hyperref}
\usepackage{cite}

\usepackage{tikz}
\usepackage{tkz-berge}
\usetikzlibrary{positioning} 

\usepackage{caption}
\usepackage{sidecap}

\usepackage{footnote}
\makesavenoteenv{tabular}

\usepackage{tcolorbox}
\usepackage{longtable}

\usepackage{setspace}
\usepackage[normalem]{ulem}

\newcommand{\new}[1]{\textcolor{black}{{#1}}}
\newcommand{\newSecond}[1]{\textcolor{black}{{#1}}}

\newcommand{\Probab}[1]{{p}\left({#1}\right)}

\newcommand{\Pcond}[2]{\Probab{{#1}\mid{#2}}}

\newcommand{\xx}{\mathbf{x}}

\newcommand{\BN}{{\sf BN}}
\newcommand{\SBCN}{{\sf SBCN}}
\newcommand{\BIC}{{\sf BIC}}
\newcommand{\AIC}{{\sf AIC}}
\newcommand{\BDE}{{\sf BDE}}
\newcommand{\BGE}{{\sf BGE}}
\newcommand{\KSCORE}{{\sf K2}}
\newcommand{\LL}{{\sf LL}}

\newcommand{\DAG}{{\sf DAG}}
\newcommand{\MLE}{{\sf MLE}}

\newcommand{\MHC}{{\sf MHC}}
\newcommand{\PPV}{{\sf PPV}}
\newcommand{\TPR}{{\sf TPR}}

\newcommand{\bth}{\boldsymbol{\theta}}

\newcommand{\bino}{\boldsymbol{\sigma}}

\newcommand{\data}{\mathbf{D}}
\newcommand{\fit}[3]{\to_{#1, #3}^{#2}}
\newcommand{\model}{{\sf M}}
\newcommand{\bootstrap}[1]{\leadsto_{#1}}
\newcommand{\reg}{{\sf f}}
\newcommand{\boot}{{\sf boot}}

\newcommand{\hnull}{{\sf null}}

\newtheorem{myth}{Theorem}[section]
\newtheorem{prop}[myth]{Proposition}

\newtheorem{mydef}[myth]{Definition}

\newcommand{\fitness}[1]{{\cal F}_{#1}}
\newcommand{\daguniv}{\mathcal{M}}
\newcommand{\dagspace}{\Pi}
\newcommand{\T}{{\sf T}}

\usepackage{cancel}
\newcommand{\added}[1]{\textcolor{black}{{#1}}}

\newcommand{\removed}[1]{}

\title{Learning the structure of Bayesian Networks via the bootstrap}

\author[1]{Giulio Caravagna\footnote{Equal contributors and corresponding authors: {\tt gcaravagna@units.it} or {\tt 
daniele.ramazzotti@unimib.it}.}}
\author[2]{Daniele Ramazzotti$^{\ast}$}
\affil[1]{Department of Mathematics and Geosciences, University of Trieste, Trieste, Italy}
\affil[2]{School of Medicine and Surgery, University of Milan-Bicocca, Milan, Italy} 
\date{}

\begin{document}
\maketitle

\begin{abstract}
Learning the structure of dependencies among multiple random variables is a problem of considerable theoretical and practical interest. Within the context of Bayesian Networks,  a practical and surprisingly successful solution to this learning problem is achieved by adopting score-functions optimisation schema, augmented with multiple restarts to avoid local optima. Yet, the conditions under which such strategies work well are poorly understood, and there are also some intrinsic limitations to learning the directionality of the interaction among the variables. Following an early intuition of Friedman and Koller, we propose to decouple the learning problem into two steps: first, we identify a partial ordering among input variables which constrains the structural learning problem, and then propose an effective bootstrap-based algorithm to simulate augmented data sets, and select the most important dependencies among the variables. By using several synthetic data sets, we  show that our algorithm yields better recovery performance than the state of the art, increasing the chances of identifying a globally-optimal solution to the learning problem, and solving also  well-known identifiability issues that affect the standard approach.  We use our new algorithm to infer statistical dependencies between cancer driver somatic mutations detected by high-throughput genome sequencing data of multiple colorectal cancer patients. In this way, we also show how the proposed methods can shade new insights about cancer initiation, and progression.
\end{abstract}

\noindent {\bf Code:} {\texttt{https://github.com/caravagn/Bootstrap-based-Learning}}

\section{Introduction}
Learning statistical structures from multiple joint observations is a crucial problem in statistics and data science. Bayesian Networks (\BN{}s) provide an elegant and effective way of depicting such dependencies by using a graphical encoding of conditional independencies within a set of random variables \cite{pearl1988probabilistic}. This enables a compact and intuitive modelling framework which is both highly explanatory and predictive, and justifies the enduring popularity of BNs in many fields of application \cite{Koeller}.

Despite the undoubtable success of \BN{}s, identifying the graphical structure underpinning a \BN{} from data remains a challenging problem \cite{chickering2004large}. The number of possible graphs scales super-exponentially with the number of nodes \cite{robinson1977counting}, effectively ruling out direct search for \BN{}s with more than a handful of nodes. Markov equivalence, the phenomenon by which two distinct graphs can encode identical conditional independence structures \cite{VernaPearl}, necessarily leads to a multimodal objective function, which can be highly problematic for {\em maximum likelihood} ({\sf ML}) optimisation-based and Bayesian methods alike. In practice, reasonable performance can be achieved by greedy methods that search  models by their likelihood  adjusted for a complexity  term \cite{gamez2011learning}. For information-theoretic scoring functions, common approaches are  \added{either} the {\em Bayesian Information Criterion} (\BIC{}) by Schwarz or the  {\em Akaike Information Criterion} (\AIC{}) by Akaike \cite{Schwarz1978,Akaike}. \added{For Bayesian scoring functions, popular choices are  the {\em Bayesian Dirichlet likelihood-equivalence score} (\BDE{}) 
\cite{cooper1992bayesian} which combines the multinomial distribution with the Dirichlet prior for discrete-valued networks, or the {\em Bayesian Gaussian equivalent} (\BGE{}) \cite{geiger1994learning}, which combines the linear Gaussian distribution with the normal-Wishart prior for Gaussian-valued networks, or the  {\em K2 score} (\KSCORE{}) \cite{cooper1992bayesian}, another  particular case of the Bayesian Dirichlet score.  All of these approaches select network structures by a greedy optimisation process, either through (regularised) optimisation of the joint parameter/ structure likelihood, or by optimising a collapsed likelihood where the explicit dependence on the conditional parameters is marginalised under a conjugate prior distribution.} As with many non-convex optimisation problems, a schema with {\em multiple initial conditions} is often used to  sample different solutions from the multi-modal fitness landscape. Nevertheless, the conditions under which they should return optimal structures are poorly understood. 

This paper presents a new  approach to the optimisation problem for \BN{} structural learning. Our method relies on simulating asymptotic conditions via a bootstrap procedure \cite{Efron}. By bootstrapping we can estimate the frequency of each edge in the model (i.e., a conditional dependence $x\mid y$), but cannot solve the Markov equivalence problem; to address that, we follow an early intuition of Koller and Friedman and devise a data-driven strategy (again based on bootstrap) to estimate a partial ordering on the set of nodes, effectively playing the role of an informative prior over graph structures \cite{Friedman,teyssier2012ordering}. Our approach therefore decouples the tasks of restricting the search space to a suitable basin of attraction, and optimising within that basin. Extensive experimentation on simulated data sets shows that the proposed algorithm outperforms several variants of regularised \added{scores} \removed{\sf ML}, and an experiment on a cancer genomics application shows how the approach can lead to insightful structure discovery on real life data science problems.

\section{Background} \label{sec:back}

In this paper we will adopt the following notation. With $\data \in \mathbb{B}^{n\times m}$ we denote the input data matrix with ${n}$ variables  and $m$ samples. For each row  a variable $\xx_i$ is 
associated, with ${\cal X} = \{ \xx_1, \ldots, \xx_n\}$. Domain $\mathbb{B}$ can be either continuous ($\mathbb{R}$, in which case we assume to be working with Gaussian conditionals) or discrete multivariate ($\mathbb{Z}$). We aim at computing a factorization of $\Probab{\xx_1, \ldots, \xx_n}$ from $\data$.  We will make use of  {\em non-parametric bootstrap} techniques \cite{Efron}: with $\data\bootstrap{k}\langle\data_1, \ldots, \data_k \rangle$ we denote $k$ non-parametric bootstrap replicates     $\data_i \in \mathbb{B}^{n\times m}$ of the input data  $\data$.

We are interested in a  {\em Bayesian Network} (\BN{}, \cite{Koeller}) $\model = \langle E,  \bth \rangle$   over  variables  $\cal X$, with edges $E\subseteq {\cal X}\times {\cal X}$ and real-valued parameters $\bth$. \new{In our formulation} $E$ must induce a \new{{\em direct acyclic graph}} (DAG) over $\cal X$, that represents  factorization 
\begin{equation}
\Probab{\xx_1, \ldots, \xx_n} = \prod_{\xx_i=1}^n \Pcond{\xx_i}{\pi_i} \qquad
\Pcond{\xx_i}{\pi_i} = \bth_{\xx_i\mid \pi_i}
\end{equation}
where $\pi_i = \{ \xx_j \mid \xx_j \to \xx_i \in E\}$ are $\xx_i$'s parents, and $\bth_{\xx_i\mid \pi(\xx_i)}$ is a {\em probability density function}. The \BN{} log-likelihood of $\model$ is given by
\begin{equation}
\LL(\data \mid \model) = \log\Pcond{\data}{E, \bth} \, .
\end{equation}

The {\em model selection} task $\data \fit{\reg}{k}{\Pi} \model_\ast$,  is to compute  a  \BN{}  $\model_\ast = \langle E_\ast,  \bth_\ast \rangle$ by solving 
\begin{equation}\label{eq:fit-score}
\model_\ast = \arg \max_{
\model = \langle E\subseteq \Pi,  \bth \rangle 
}
\LL(\data \mid \model) - \reg(\model, \data)
\end{equation}
where $\reg$ is a {\em regularization score} \cite{Koeller} \added{(e.g., \BIC{}, \AIC{}, \BDE{}, \BGE{}, \KSCORE, etc.)}; \new{notice that in this formulation we are implicitly assuming that  the graph induced by the selected edges $E\subseteq \Pi$ is acyclic, i.e., a \DAG{}\footnote{\new{This model selection problem is formally defined on the space of \DAG{}s; therefore $E$ in equation (\ref{eq:fit-score}) should be constrained to a valid \DAG{}. The set $\Pi$, from which the final graph $E\subseteq \Pi$ is selected, can contain an arbitrary set of edges (i.e. also a set of edges that induce cycles), and it is a requirement of the model-selection heuristic to ensure that the selected edges $E$ induce a \DAG{}.}}.} This problem is {\sf NP-hard} and, in general, one can compute a  (local) optimal solution to it \cite{chickering2004large}. 
In our definition the search-space  is constrained  by $E \subseteq \Pi$. Without loss of generality, we assume $\model_\ast$ to be  estimated by  a {\em hill-climbing}  procedure that starts from $k$  random initial \BN{}s, and returns  the highest scoring model.  When one uses information-theoretic scoring functions, parameters are {\em maximum-likelihood estimates} (\MLE{})  of the conditional distributions\footnote{If $\model_\ast$ is categorical with $w$ values, then the multinomial estimate is 
\[
\bth^{\sf ML}_{\xx_i = x\mid \pi_i=y} = \dfrac{n(x, y)}
{\sum_{\xx_i=v_1}^{v_w} n(v_i, y)}\, ,
\]
where $n(\xx_i, y)$ counts, from $\data$, the number of observed instances for an assignment of $\xx_i$ and $y$.} \cite{Koeller}. 

We will make use also of weighted \DAG{}s whose definition is standard;  $w_E(\xx_i \to \xx_j)$ will be the weight associated to edge $\xx_i \to \xx_j$ in a graph with edges $E$ via function $w: E\to \mathbb{R}$. 

\paragraph{\removed{Standard} \added{Baseline} approach.}

In what follows we will aim at improving over the \removed{standard} \added{baseline} approach, which we consider to be the $\reg$-regularized  selection with unconstrained search space and $k$  initial conditions
\[
\data \fit{\reg}{k}\emptyset \model_\ast\, .
\]
This procedure is {\em greedy}, it starts from an initial condition $\model_0$ -- e.g., a random \DAG{} -- and performs a one-edge change (deletion or insertion of an edge) to exhaustively compute the {\em neighbourhood} ${\cal M}_0$ of $\model_0$. Then, 
 $\hat{\model} \in {\cal M}_0$ is the new best solution if it has score -- according to equation (\ref{eq:fit-score}) -- higher than $\model_0$ and is the maximum-scoring  model in the whole neighbourhood. The greedy search then proceed recursively to examine $\hat{\model}$'s neighbourhood, and stops if the current solution is the highest scoring in all of its neighbourhood. Thus, this search  scans a set of solutions $\{\model_i\}_I$ by maximising the {\em discrete gradient} defined as
\begin{align}\label{eq:gradient}
\nabla_{\model_i, \hat{\model}} =  f(\hat{\model}) - f(\model_i),&& \hat{\model} \in {\cal M}_i
\end{align}
 where  $f(\model) = \LL(\data \mid \model) - \reg(\model, \data)$  is the scoring function in equation (\ref{eq:fit-score}).

\added{Hill Climbing  is known to be suboptimal, and can be improved in several ways. For instance, instead of sampling $k$ uncorrelated initial conditions \new{({\em random restarts})},  \new{one can  sample a model in the neighbourhood of the last computed solution and proceed through an {\em iterated local search}}. \new{To navigate iteratively the space of solutions}  one can  take into account  structural-equivalence classes,   node orderings, and   edge reversal moves; see \new{\cite{chickering2002learning,friedman2003being,grzegorczyk2008improving,goudie2016gibbs,kuipers2017partition}} and references therein. \new{Other approaches can guarantee exact Bayesian structure learning by applying either dynamic programming \cite{koivisto2004exact,eaton2007exact} or integer linear programming, \cite{cussens2012bayesian,cussens2017bayesian}} nevertheless, the number of   valid solutions remains still potentially huge \new{and Markov equivalence still poses challenges}. \newSecond{We refer to \cite{scutari2019learns,scanagatta2019survey} for recent reviews of approches for structure learning of Bayesian Networks}. For simplicity, here we consider  the baseline Hill Climbing; it would  be straightforward to improve our approach by adopting other search or restart strategies  proposed in the literature.}

\added{In this paper we consider  several common scores for \BN{}s: the \BIC{}, \AIC{}, \BDE{}, \BGE{} and \KSCORE. In the Main Text, we discuss results obtained with the information-theoretic scores ${\sf f} \in \{\BIC, \AIC\}$; $\BIC$ is derived as the  infinite samples approximation to the \MLE{} of the  structure {\em and}  the parameters  of the model, and is consistent, while \AIC{} is not.  In the Supplementaty Material, we present that analogous results hold for Bayesian scoring functions (${\sf f} \in \{\BDE{}, \BGE{}, \KSCORE\}$).}

\new{Searching for the optimal network requires also to account for the fact that different \DAG{}s can induce the same distributions; this is formalised through the notion of $v$-structures and likelihood equivalence, which are structural properties of \BN{}s introduced in Section \ref{sec:example} (together with one example). Intuitively, if we denote by $K_\model$ the set of likelihood-equivalent models, even in the case of infinite samples ($m \to +\infty$) asymptotic convergence is up to Markov equivalence. This mean, in practice, that we can at best identify one of the models in the equivalence class $K_{\model_\T}$, not necessarily the true one, and therefore {\em the fitness landscape is multi-modal, each mode being one of the elements of $K_{\model_\T}$}.  For finite  $m$, model-selection is even more complicated. The landscape induced by the likelihood function is rugged and there could $i)$ be structures   scoring higher than the ones in $K_{\model_\T}$, and $ii)$ also higher than the models in their neighbourhood (thus suggesting the importance of testing also randomised restarts). Thus, such structures {\em as well as} their equivalence classes  would create further optima; we present one example of this models in  Section \ref{sec:example}, and a portrait the associated multi-modal landscape in Figure \ref{fig:landscape2}.  For this  reason, besides the problem of identifying the one true model $\model_\T$  {\em within} $K_{\model_\T}$, a greedy search is likely be trapped into local optima, and  heuristics  use multiple initial conditions to minimize such an effect.  }


\section{An example} \label{sec:example}

We give an intuitive introduction to the concept of fitness landscape associated with this optimisation problem, and show its computation on a real network.
\begin{mydef}[Fitness]\label{def:fitness}
Consider $\daguniv \subset {\cal X} \times {\cal X}$  the set of all possible non-reflexive edges over variables in set $\cal X$. For a subset $\dagspace \subseteq \daguniv$, let $\fitness{\dagspace,\reg}: 2^\daguniv \mapsto \mathbb{R}^+$ be the {\em fitness function} of the  state space $2^\dagspace$, data $\data$ and regularization $\reg$ and the \BN{} $\model  = \langle E, \bth \rangle$ to be defined by 
\begin{equation}
\fitness{\dagspace,\reg}(E) = 
\begin{cases}
\LL(\data \mid \model) - \reg(\model, \data)\, ,  &\text{if $E\subseteq  \dagspace$, $E$ acyclic,}\\
0 & \text{otherwise.}
\end{cases}
\end{equation} 
Then, $\fitness{\dagspace,\reg} (\cdot)$ defines the {\em fitness landscape}  which  we use to search for a \BN{}  model $\model^\MLE = \langle E^\MLE, \bth^\MLE \rangle$ that best explains $\data$ in the  sense of equation (\ref{eq:fit-score}).
\end{mydef}
So, in practice, a search that  constraints the state space  by $\Pi$ spans through the subspace of \DAG{}s induced by $2^\dagspace \subseteq 2^\daguniv$. Let us denote the {\em true model} as the \BN{} $\model_\T = \langle E_\T, \bth_\T\rangle$, $E_\T \in  2^\dagspace$;  for $m \to \infty$, the   landscape's  \MLE{} structure is $E_\T$, when  $\reg$ is a {\em consistent estimator} ($\BIC$ does satisfy this property,  if at least one of several models contains the true distribution \cite{haughton1988choice}). 
Unfortunately, the \MLE{} is not unique even for infinite sample size.
\begin{prop}[Likelihood equivalence \cite{Koeller}, Figure \ref{fig:landscape2}]\label{def:eqcl} For any \BN{}  $\model  = \langle E, \bth \rangle$
there exists  $K_\model = \{\model_i = \langle E_i, \bth_i \rangle\}_I$ for some index set $I$, such that 
$\fitness{\daguniv,\reg} (E_i) = \fitness{\daguniv,\reg} (E_j)$ for every  $\model_i, \model_j \in K_\model$.
\end{prop}
We term $K_\model$  a {\em Markov equivalence (or I-equivalence) class} of \BN{}s  with  equivalent fitness value, but different structure. Thus, we can not  expect to identify $\model_\T$ among $K_{\model_\T}$'s models by looking at  $\fitness{\daguniv,\reg} (\cdot)$, which leaves us with, at least, $|K_{\model_\T}|$ equivalent maxima. Such class exists  due to symmetries of the likelihood function that are induced by    $v$-structures.
\begin{mydef}[$v$-structure \cite{VernaPearl}, Figure \ref{fig:landscape2}]
A triplet  $(\xx_i, \xx_j, \xx_k)$ is a {\em $v$-structure} in a set of edges $E$ if  
 $\xx_i \to \xx_k, \xx_j \to  \xx_k \in E$ but $\xx_i \to \xx_j, \xx_j \to \xx_i \not\in E$.
\end{mydef}

\begin{figure}[t]
\centerline{\includegraphics[width = 1\textwidth]{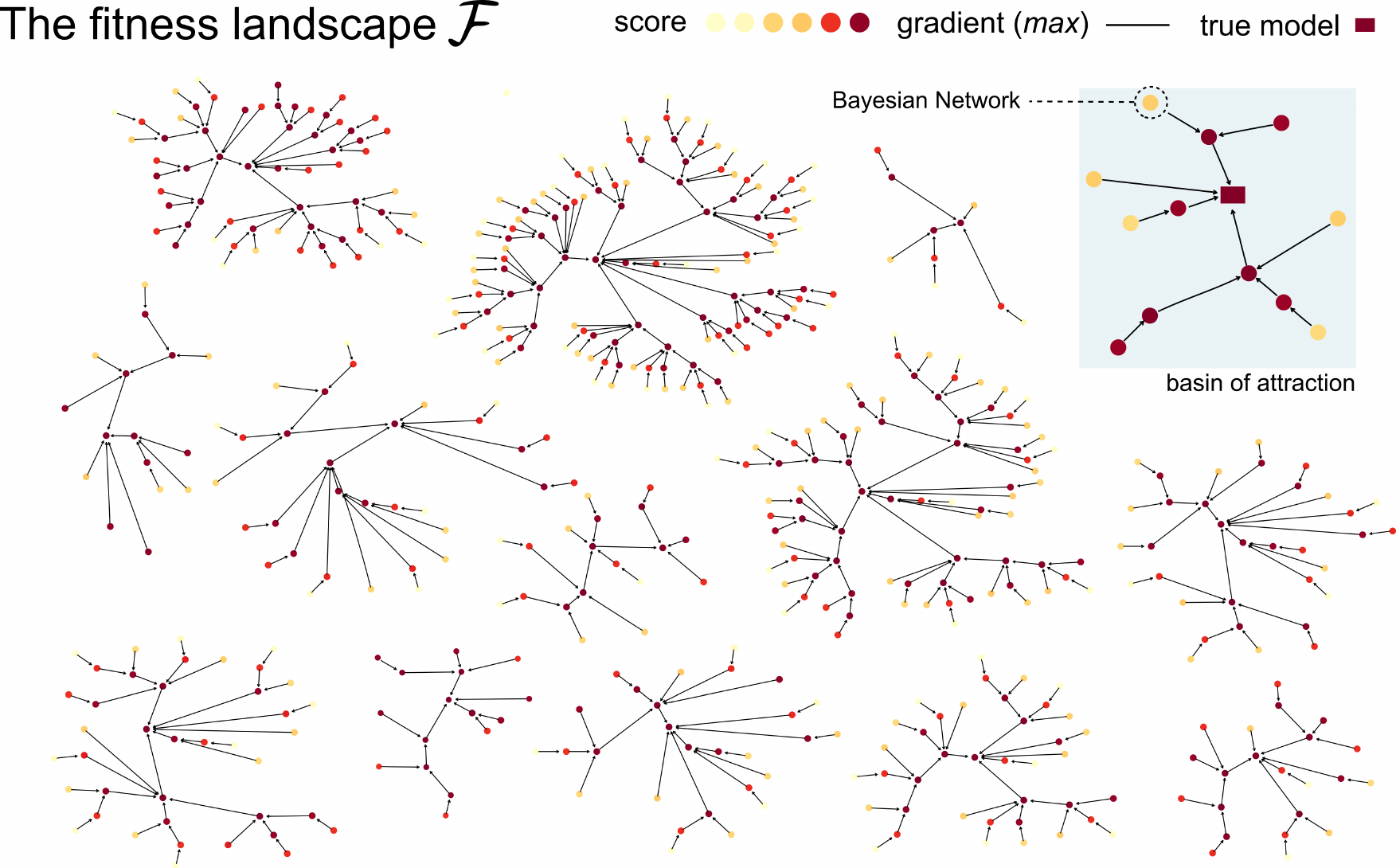}}
\caption{Exhaustive portrait of the fitness landscape $\fitness{}$ (Definition \ref{def:fitness}) for a random \BN{} with $n=4$ variables, random conditional distributions $\bth$ and $10000$ samples. \added{The scoring function uses  \BIC{}}. Each node is a candidate \BN{}, whose score is given by the color's intensity (darker is better). In total, there are $543$ \BN{}s. Each edge represents the maximum of the optimization gradient in equation (\ref{eq:gradient})S, which is followed by a greedy heuristics such as Hill Climbing. Here the neighbourhood of a model is the set of models that differs by one  edge. A basin of attraction is a set of initial conditions that lead to the same solution. Here the true model  is associated to a mid-size basin of attraction,  highlighted in top right of the plot. In Figure \ref{fig:landscape2} we show the local optima,  the true model and a way to re-shape $\fitness{}$.
}
\label{fig:landscape1}
\end{figure}

\begin{figure}[t]
\centerline{\includegraphics[width = 1\textwidth]{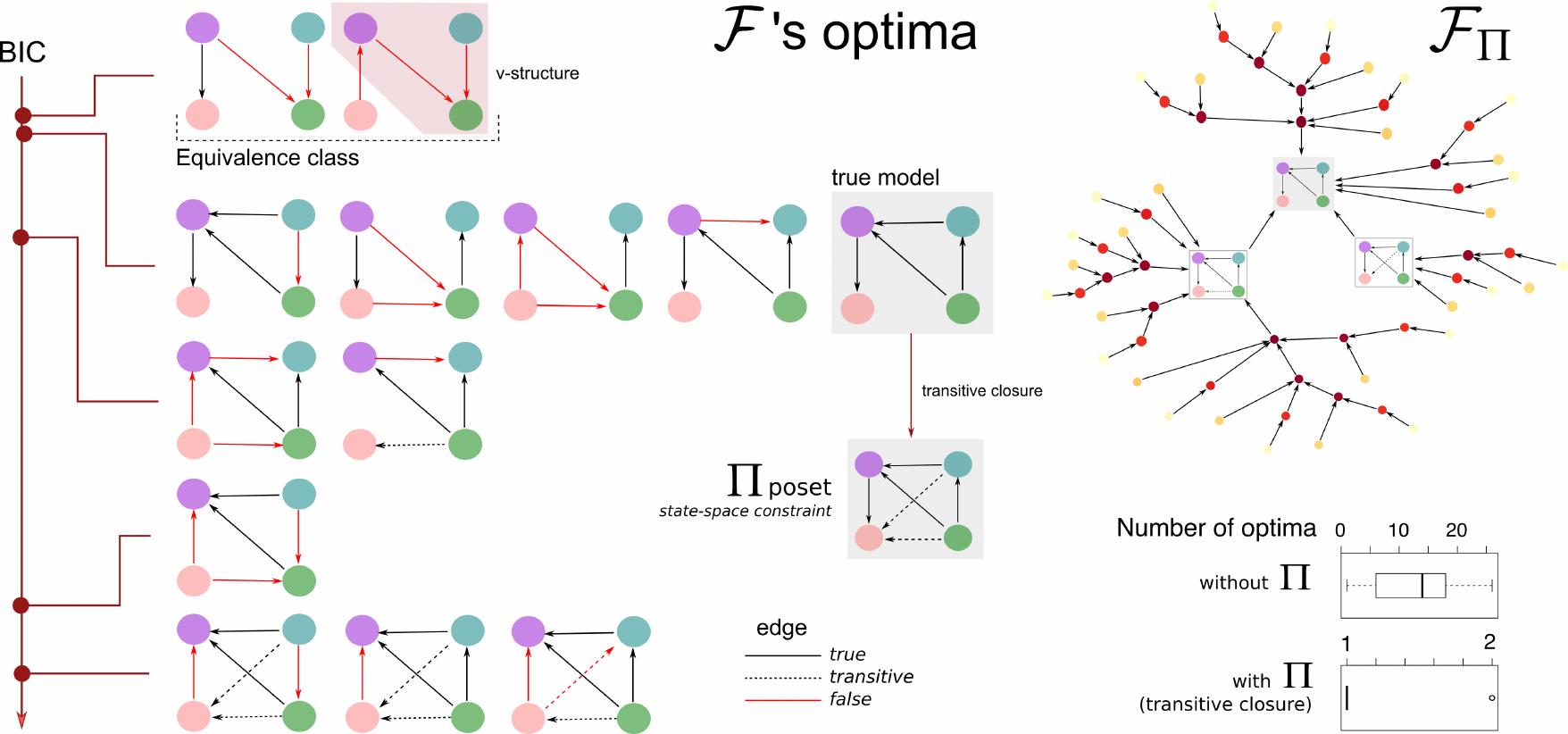}}
\caption{The 13 optima of the fitness landscape shown in Figure \ref{fig:landscape1}, with their \BIC{} score. Notice the equivalence classes (discussed in Section \ref{sec:example}) and the presence of optima with equivalent score but different structure. The true, i.e., generative, model is not the highest ranked in $\fitness{}$. If we create as poset $\Pi$ the transitive closure of the true model, however, we observe that the landscape reduces to having a {\em unique global optima}. In fact, all the optima but the true one have at least one edge not included in $\Pi$. For this $\Pi$, the landscape happens to be {\em unimodal with a maximum at the true model}; an  experiment with $100$ random networks shows that this happens with high probability.}
\label{fig:landscape2}
\end{figure}

\paragraph{Example with a simple network.}
We begin with an example that inspired the approach that we introduce in Section \ref{sec:appr}.  Let us consider a random \BN{} $\model$ with  $n=4$ discrete nodes (${\cal X} = \{\xx_1, \ldots,\xx_4\}$, $\mathbb{B}=\{0,1\}$), $|\pi_i|\leq2$, and \added{random  conditional distributions $\bth$ (parameters)}. Despite being small, models of this size show a rich optimization's landscape  and allow for some visualization. In fact, the number of \DAG{}s with $n$ nodes is super-exponential in $n$. Precisely, it is computable as $$G(n) = \sum_{k=1}^n (-1)^{k+1} {n\choose k} 2^{k(n-k)} G(n-k)$$ as shown in \cite{robinson1977counting}; in this case leads to $G(n)=543$ models.  From $\model$, we generate  $m=10000$  samples and investigate the problem of identifying $\model$ from such  data. 

With such a small network we can  exhaustively construct the fitness landscape $\fitness{}$ of the discrete optimization, and visualize the gradient in equation (\ref{eq:gradient}) used to solve equation (\ref{eq:fit-score}). The whole landscape of the Hill Climbing with  \BIC{}  scores is shown in Figure \ref{fig:landscape1}, and shows  that: 
\begin{itemize}
\itemsep0em 
\item[$(i)$] there are several models with different structure but equivalent \BIC{} score;
\item[$(ii)$]  $\model$'s \BIC{} score is not the highest in this  landscape, which has 13 optima;
\item[$(iii)$] the basins of attractions can be fairly large, compared to $\model$'s one;
\end{itemize}
We expect the landscape to have multiple modes because of Markov equivalence classes (Definition \ref{def:eqcl}), and because we are working with finite $m$. Thus, a search in this landscape could likely be trapped in optima that are not $\model$.

We now focus on the intuition that searching for the model is generally easier if one constrains the parent sets   \cite{Koeller}.  This is often done by either setting a cutoff on $|\pi_i|$, i.e., limiting the number of $\xx_i$'s parents, or by specifying  a {\em partially ordered set} ({\em poset}) $\Pi \subseteq {\cal X} \times {\cal X}$ such that $\xx_j$ can be one of $\xx_i$'s parents only if $(\xx_j, \xx_i) \in \Pi$. Whatever the case, the algorithmic motivation seems obvious as we drop the search's combinatorial complexity  by pruning possible solutions. However, we are interested in {\em investigating how  this  affects the shape of the landscape $\fitness{}$}.

We consider the constraint to be   given as a  poset $\Pi$ (that, in practice, one has to estimate from data). The search is then limited to analyzing edges  in $\Pi$, so $\Pi$  is good  if it shrinks the search to visit solutions that are ``closer'' to $\model$ -- thus, $\Pi$ has to include  $\model$'s edges. In this example  we create $\Pi$    by adding  to $\model$  also  its transitive edges. In Figure \ref{fig:landscape2} we show that all the models (but $\model$) that are optima in $\fitness{}$ have {\em at least one edge} that is not allowed by $\Pi$. So,  they {\em would not} be visited by a search constrained by this $\Pi$.

We compute the fitness landscape under $\Pi$, $\fitness{\Pi}$, and find it to have a {\em unique  optimum} (Figure \ref{fig:landscape2}). For this poset, $\fitness{\Pi}$ is unimodal with a maximum at the true model. $\model$'s basin of attraction in $\fitness{\Pi}$  is  larger than in $\fitness{}$, as one might expect. This clearly suggests that we are also enjoying a simplification of the ``statistical part'' of the problem, which we observe with high probability ($98$ times out of $100$) in a sample of   random networks. In two cases, we observed   two optima in  $\fitness{\Pi}$ ($\model$ and one of its subsets, data not shown). Thus, {\em greedy optimization} of equation (\ref{eq:fit-score}) in this setting would lead to the globally optimal solution $\model$.


The above considerations are valid for the $\Pi$ derived  as transitive closure of $\model$. In real cases, of course, we do not know $\model$ and cannot trivially build this $\Pi$.  We can, however, try to   approximate  $\Pi$ from $\data$. In practical cases, of course, the landscape  will still be multi-modal under the approximated poset, but one would hope that the number of modes is reduced and the identification of the true model made easier in the reduced search-space.

\newcommand{\mstar}{{\textcolor{BrickRed}{($\star$)}}}

\begin{algorithm}[t]                      
\setstretch{0.25}
\caption{-- {\bf Model selection for \BN{}s via the bootstrap (Figure \ref{fig:algorithm}.)} \\\footnotesize Steps marked with \mstar{} can be implemented in different ways (see Sections \ref{sec:loops}--\ref{sec:htest}).}          
\label{alg1}                           
\begin{algorithmic}[1]                    
    \REQUIRE a dataset $\data$ over variables $\cal X$, and two integers $k_p$, $k_b \gg 1$;
    \STATE let  $\data\bootstrap{k_p}\langle\data_1, \ldots, \data_{k_p} \rangle$ \new{be $k_b$ bootstrap resamples from $\data$}, and  $\daguniv \subset {\cal X} \times {\cal X}$ be the set of  non-reflexive edges over  $\cal X$.
    \STATE compute the {\em weighted consensus structure} $\Pi_{\boot}$
    \begin{align}\label{eq:firstboot}
\Pi_{\boot} &= \bigcup_{i=1}^{k_p} \Big\{ E_i  
\mathrel{}|\mathrel{}
  \data_i \fit{\reg}{1}{\daguniv} \model_i = 
\langle E_i, \bth_i\rangle\Big\} &&
w_{\Pi_\boot}(\xx_i \to \xx_j) = \sum_{w=1}^{k_p} \mathbf{1}_{E_w} (\xx_i \to \xx_j)\, ;
\end{align}
\new{where by $\data_i \fit{\reg}{1}{\daguniv} \model_i$ we mean to learn the BN $\model_i$ from the bootstrap sample $\data_i$ using a single run,  $\reg$-regularisation and scanning all possible edges ($\daguniv$)};
\STATE \mstar{} remove loops from $\Pi_{\boot}$ by solving 
\begin{equation}\label{eq:subsetboot}
\Pi = \arg
\max_{\substack{\Pi_\ast \subseteq \Pi_\boot \\ \Pi_\ast \text{ \sf acyclic}}}\quad
 \sum_{\xx_i \to \xx_j \in \Pi} w_{\Pi}(\xx_i \to \xx_j)\,;
\end{equation}
\STATE let $\data\bootstrap{k_b} \langle\data_1, \ldots, \data_{k_b} \rangle$, for any $\data_i$ generate $\hat{\data}_i  = {\sf perm}(\data_i)$;
\STATE compute $2k_b$ \BN{}s under $\Pi$
\begin{align}
\Gamma &= \{  {E}_i \mid    \data_i \fit{\reg}{1}{\Pi} \model_i = 
\langle  {E}_i, \bth_i\rangle\} && 
\Gamma_\hnull = \{  {E}_i \mid  \hat{\data}_i \fit{\reg}{1}{\Pi} \model_i = 
\langle  {E}_i, \bth_i\rangle 
\}\, ,
\end{align}
\new{Note that here we use $\Pi$ to constrain the search space for each \BN{}};
\STATE let 
$\bino_{i,j} = [\cdots\mathbf{1}_{x}(\xx_i \to \xx_j)\cdots]_{x \in \Gamma}$ and $
\bino_{i,j}^{\hnull} = [\cdots\mathbf{1}_{x}(\xx_i \to \xx_j)\cdots]_{x \in \Gamma_\hnull}$;
\STATE \mstar{}  to select $\xx_i\to \xx_j$, test $H$ at level  $\alpha$ with    Multiple Hypotheses Correction (\MHC{}) and
 output the Bayesian Network  $\model=\langle E, \bth^\MLE \rangle$ where
\begin{align}
E &= \{ \xx_i \to \xx_j \mid H: \mathbb{E}[\bino_{i,j}] \neq_\alpha \mathbb{E}[\bino_{i,j}^\hnull] \} &&
\theta^{\MLE{}} = \arg\max_{\bth \in \Theta} \log\Pcond{\data}{E, \bth}\, .
\end{align} 
\end{algorithmic}
\end{algorithm}

\begin{figure}[t]
\centerline{\includegraphics[width = .9\textwidth]{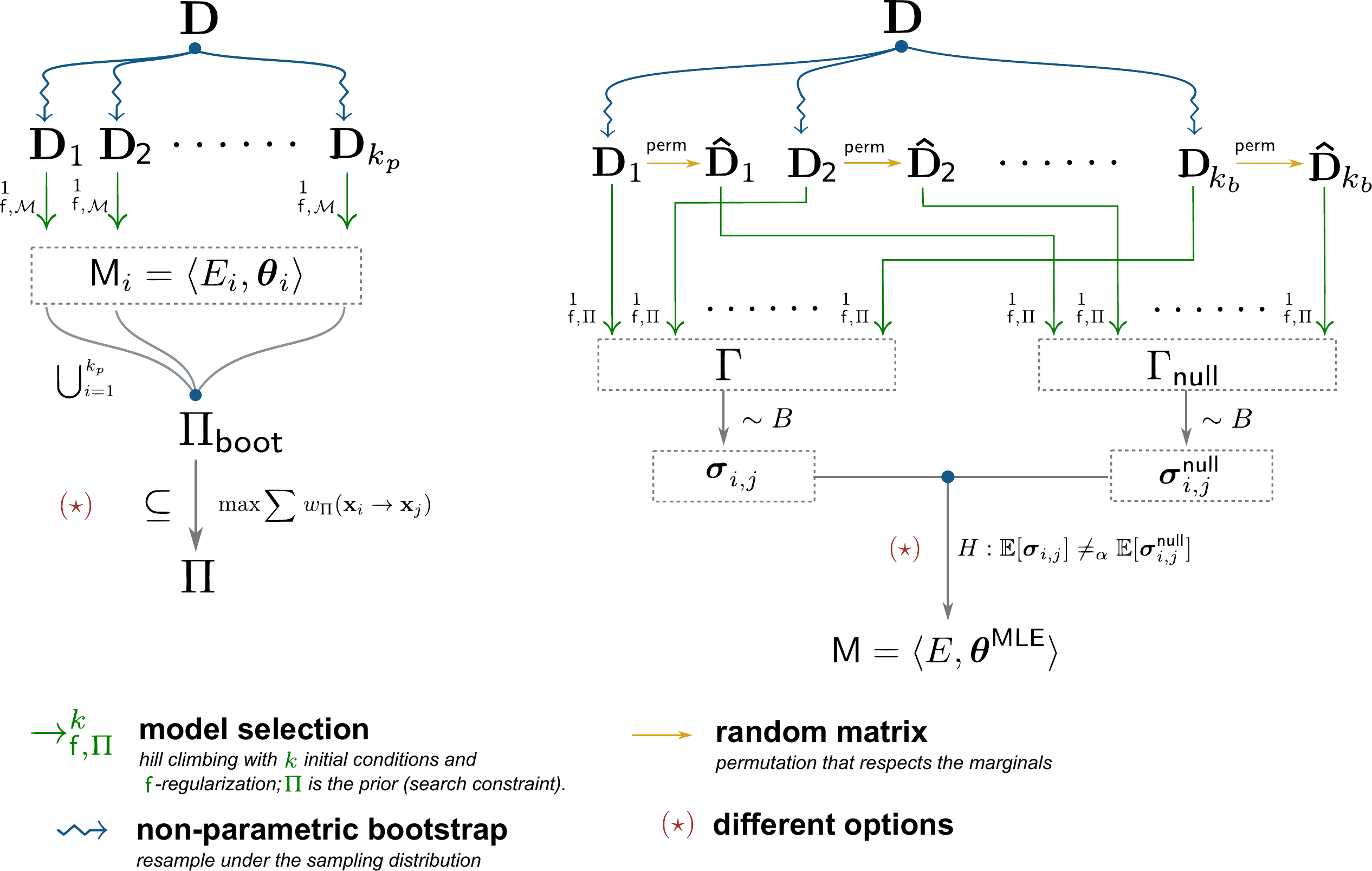}}
\caption{Graphical representation of Algorithm \ref{alg1}. Left: first phase (construction of the poset $\Pi$). Right: second phase  (construction of the test under the poset $\Pi$).
}
\label{fig:algorithm}
\end{figure}

\section{Model selection for \BN{}s via empirical Bayes} \label{sec:appr}

We present our method as Algorithm \ref{alg1}; the algorithm exploits  a combination of {\em non-parametric bootstrap estimates}, {\em likelihood-fit} and {\em hypothesis testing} to infer a \BN{} . The algorithm is  conceptually divided in two phases (Figure \ref{fig:algorithm}) that  can be customized, as we discuss in the next subsections.

\paragraph{Phase one: construction of the poset $\Pi$.}
 The first phase (steps 1--3) uses a bootstrap strategy to estimate an ordering $\Pi$ of the model's variables; this ordering  constraints the factorization in the second phase of the algorithm. \new{The bootstrap is used in the following way. For $k_p$ times we sample with repetition a dataset of equal size with respect to the input dataset $\data$ -- i.e., this is a classic non-parametric bootstrap scheme. For each bootstrap sample we run the standard model selection strategy: i.e., we denote by $\data_i \fit{\reg}{1}{\daguniv} \model_i$ the learning of the model $\model_i$ from the bootstrap sample $\data_i$ using a single run (no restarts),  $\reg$-regularisation and scanning all possible edges ($\daguniv$) to create the model. This steps practically creates $k_p$ models.}
  
The union $\Pi_{\boot}$ of all the $k_p$ models' structures  is obtained by merging all the fits from the non-parametric bootstrap replicates. \new{This is a trivial graph union operation which, of course, does not necessarily preserves the acyclic condition required by a \BN{}.} This structure is called consensus as it contains the union of all the models that are obtained by a standard regularized likelihood-fit procedure. Notice that each model  is obtained from one initial condition,  and without restrictions on the set of candidate edges that can populate the models\footnote{\new{In our implementation of the algorithm we use the default initial condition of package  {\sf bnlearn} \cite{bnlearn} to  determine by hill-climbing  the fit of each bootstrap resamples; this is the empt model without edges. Of course,  this initial condition can be generated by using different strategies such as random sampling, or correlated initial conditions.}}. Each models' parameters \new{(i.e., the conditional probability tables)} are dropped, and  $\Pi_{\boot}$ is instead augmented  with the {\em non-parametric bootstrap scores} via   the set indicator function $\mathbf{1}_X(y) = 1 \iff y \in X$. \new{This is just a way of counting how often each edge is detected across the $k_p$ bootstrap resamples;} thus $w_{\Pi_\boot}(\cdot)$ is proportional to the edges' frequency  across the $k_p$ bootstrap models. 
 
 The graph induced by $\Pi_{\boot}$ is generally cyclic, and is weighted. In step 3, we  render it  acyclic by selecting a suitable subset of its edges: $\Pi \subseteq \Pi_{\boot}$. This loop-breaking strategy is based on the idea of  maximizing the scores of the edges  in $\Pi$, and is motivated by the intuition that true model edges should have  higher bootstrap scores\cite{Friedman}. The optimization problem that determines $\Pi$, equation (\ref{eq:subsetboot}), can be solved in different ways, as we discuss in Section \ref{sec:loops}.

\paragraph{Phase two: using $\Pi$ to construct a final  model.}

The second phase (steps 4-7)  is the actual selection of the final output model. In principle, we could just  use the standard regularized likelihood-fit procedure to select a model under $\Pi$\footnote{\newSecond{This would be equivalent} to the model selection strategy adopted to process the bootstrap samples, with $\Pi$ used as constraint for the set of edges that can be used to populate the model.}. Preliminary tests (data not shown), however, have highlighted an intrinsic bias\footnote{Precisely, we observed that if we here proceed by selecting a model via likelihood-fit, the variance in the estimated solution will be small and consistent with the choice of the regularization function $\reg$  -- e.g., \BIC{} would select sparser models than \AIC{} -- regardless how good is our estimate of $\Pi$ (i.e., how likely is that $\Pi$ contains all the true model edges). We term this the phenomenon ``intrinsic bias'' of the regularizer.} in the selected ouput model, as a function of the   regularizer  $\reg$.  We would like to reduce to the minimal extent this effect, while enjoying the properties of $\reg$ to minimize overfit.  Thus,  we exploit $\Pi$ to create  an edge-specific statistical test to detect true edges, and create the final output model.  Here,  if $\Pi$ is a good approximation to the transitive closure of the true model (such as in the example of Section \ref{sec:example}), then $\Pi$ will direct the search to \removed{getter} \added{get} better estimates for the test; therefore \new{in this case, an approximation is ``good'' if it contains all the true model edges}.

The test {\em null hypothesis} $H_0$  is created from $\data$, again by exploiting a bootstrap procedure. We begin by creating  (step 4)   $k_b$ bootstrap resamples of $\data$\new{, as in step one of the algorithm}; from each replicate we  generate a {\em permutation matrix} $\hat{\data}_i \in \mathbb{B}^{n\times m}$, with equivalent  empirical marginal distributions. \new{The construction of the matrix depends on the type of distributions that we are modelling; let $p_{i}(\xx_j)$ and  $\hat{p}_i({\xx}_j)$ be the empirical marginals of $\xx_j$ in $\data_i$ and $\hat{\data}_i$. If $\xx_j$ is discrete multivariate we require  $p_{i}(\xx_j) = \hat{p}_i({\xx}_j)$. If $\xx_j$ is continuous, we require the expectation and variance to be equivalent. We achieve this with a shuffling approach: we  independently permute $\data_i$'s row vectors -- in the algorithm denoted by function {\sf perm$(\cdot)$}.}  The joint distributions in each $\hat{\data}_i$ are  random, so for each pair $(\xx_i, \xx_j)$  we have a null model of their statistical independence  normalized for their marginal distributions. \new{At this point we have a pair of $2k_b$ datasets, half of them are bootstrap resamples, and the other half are matched permutation datasets. We can now fit a model for each on of these datasets; in this case we use the same fitting strategy adopted in the first step of the algorithm, but constraining the edges to include in the mode by using the poste $\Pi$. The obtained $2k_b$ models are split into two groups depending on the data we used to generate them (non-permuted versus permuted); the two groups are called  $\Gamma$ and $\Gamma_\hnull$ (the models from the null hypothesis), and the edges are counted as in the step one of the algorithm. Thus, if we fit a model on ${\data}_i$ and $\hat{\data}_i$ (step 5) we expect that an edge that represents a true dependency will tend to be more often present in $\Gamma$, rather than in $\Gamma_\hnull$.}

Steps 6 and 7 perform {\em multiple hypothesis testing} for edges' selection.  We use the \removed{\MLE{}} models computed in step 5 as a proxy to test for the dependencies.  The vectors  $\bino_{i,j}$ and $\bino_{i,j}^\hnull$   store how many times  $\xx_i \to \xx_j$ is detected  in $\Gamma$ and $\Gamma_\hnull$, respectively, so each $\bino_{i,j}$ is a sample of a Binomial random variable over $k_b$ trials. Then, we can carry out a Binomial test (or, if $k_b$ is large,  a 2-sided T-test) with confidence $\alpha$ and corrected for multiple   testing. We will  include every accepted edge $\xx_i \to \xx_j$ in the final output model $\model$, augmented with the \MLE{} of its
 parameters (estimated from the original dataset $\data$). Notice that $\model$ is acyclic as, by construction, $\Pi$ is acyclic.

\paragraph{Complexity analysis.} Our procedure has cost dominated by the computation of the bootstrap estimates and  likelihood-fits. \new{In particular, for any single run of fits by  hill-climbing, the same performance and scalability of standard hill-climbing implementations is to be expected (for that run).} However, we note that our algorithm \new{has a design} that allows for a simple parallel implementation to compute each estimate (i.e., bootstrap resample and its likelihood-fit). 
\new{This seems particularly advantageous considering the steady drop for the cost of parallel hardware such as high-performance clusters and graphical processing units}.  Once all estimates are computed, the cost of loop-breaking is proportional to the adopted heuristics, and the cost of multiple hypothesis testing is standard.

\subsection{Removing loops from $\Pi_\boot$} \label{sec:loops}

The problem of determining a \DAG{} (here $\Pi$) from a directed graph with cycles (here $\Pi_\boot$) is  well-known  in graph theory \cite{karp1972reducibility}. This problem consists in
detecting a set of edges which, when removed from the input graph, leave a \DAG{} -- this set of edges is called {\em feedback edge set}.

In Algorithm \ref{alg1}  edges in  $\Pi$ will  constrain the search space, so it seems reasonable  to remove as few of them as possible. Since  the input graph is weighted by the non-parametric bootstrap coefficients, we can also interpret the cost of removing one edge as proportional to its weight.    Thus,  we need to figure out the minimum-cost  edges to remove, which corresponds to the  {\em minimum feedback edge set} formulation of the problem.
In general, this problem is {\sf NP-hard} and several approximate solutions have been devised (see, e.g., \cite{MFAS}).

We propose two different strategies to solve the optimization problem in equation (\ref{eq:subsetboot}) which are motivated by practical considerations. 

\begin{enumerate}
\item {\sf (confidence heuristic).}  An approximate solution to the problem can be  obtained by a greedy heuristics that breaks loops according to their weight $w_{\Pi_\boot}$. The approach is rather intuitive: one orders all the edges in $\Pi_\boot$ based on their weight -- lower scoring edges are considered first. Edges are then scanned in order according to their score and removed if they cause any loop in $\Pi_\boot$. This approach is, in general, sub-optimal. 

The algorithmic complexity of the method depends first on sorting the edges and on the subsequent loop detection. Given a number of $a$ edges in $\Pi_\boot$, they can be sorted with a sorting algorithm, e.g., quicksort \cite{hoare1962quicksort}, with a worst case complexity of $\mathcal{O}(a^2)$ (average complexity for quicksort $\mathcal{O}(a\log a)$). Then, for each ordered edge, we evaluate loops, e.g., either by depth-first search or breadth-first search (complexity $\mathcal{O}(n + a)$, with $n$ being the number of vertices\cite{cormen2009introduction}). This leads to a total complexity of $\mathcal{O}(a^2)$ + $\mathcal{O}(n + a)$ in the worst case for removing the loops. 

\item {\sf (agony).} In \cite{gupte2011finding}, Gupte {\em et al.} define a measure of the hierarchy existing in a directed graph. Given a directed graph $G = (V,E)$, let us consider a ranking function $r: V \rightarrow \mathbb{N}$ for the nodes in $G$, such that $r(u) < r(v)$ expresses the fact that node $u$ is ``higher'' in the hierarchy than $v$. If $r(u) < r(v)$, then  edge $u \rightarrow v$ is expected and does not cause any ``agony''. On the contrary, if $r(u) \geq r(v)$ edge $u \rightarrow v$ would cause agony.

We here remark that any \DAG{} induces a partial order over its nodes, and, hence, it has always zero agony: the nodes of a \DAG{} form a perfect hierarchy. Although the number of possible rankings of a directed graph is exponential, Gupte {\em et al.} provide a polynomial-time algorithm for finding a ranking of minimum agony. In a more recent work, Tatti {\em et al.} \cite{tatti2015hierarchies} provide a fast algorithm for computing the agony of a directed graph. With $a$ being the number of edges of $G$, the algorithm has a theoretical bound of $\mathcal{O}(a^2)$ time. 

Therefore, we can compute a ranking over $\Pi_\boot$ at minimum agony, i.e., a ranking of the nodes with small number of inconsistencies in the bootstrap resampling, thus which maximizes the overall confidence. With such a ranking, we can solve equation (\ref{eq:subsetboot}) by 
 removing from $\Pi_\boot$ any edge which is inducing  agony. 
\end{enumerate}


\begin{prop}\label{ref:prop-posets} The poset  $\Pi$ built by {\sf agony} is a superset of the one computed by {\sf confidence heuristic}. See Figure \ref{fig:synth}.
\end{prop}

\newcommand{\FDR}{{\sf FDR}}
\newcommand{\FWER}{{\sf FWER}}

\subsection{Multiple hypothesis testing} \label{sec:htest}

{\em Correction for Multiple Hypotheses Testing} (\MHC{})  can be done in two ways: one could correct for  {\em false discovery rate} (\FDR{}, e.g., Benjamini-Hochberg) or {\em family-wise error rate} (\FWER{}, e.g.,  Holm-Bonferroni). The two   strategies have different motivation: \FWER{}   corrects for the probability of at least one false positive, while \FDR{} for the proportion of false positives among the rejected null hypotheses.  Thus, \FWER{} is a stricter correction than \FDR{}. 

Given these premises, it is possible to define a rule of thumb. If one has reason  to believe that $\Pi$ is ``close'' to the true model, i.e., $\Pi$  has few false positives, then a less stringent correction such as \FDR{} could be appropriate. Otherwise, a \FWER{}   approach might be preferred.

Multiple hypotheses testing is also influenced by the number of tests that we carry out. We  perform  $|\Pi|$ tests, and hence \FWER{} scales as $\alpha/|\Pi|$. The  theoretical bound \removed{to} \added{on} $|\Pi|$ is  the size of the biggest direct acyclic graphs over $n$ nodes
 \begin{align}
|\Pi| \leq \left(\sum_{i=0}^n n-i\right) - n = \dfrac{n(n+1) - 2n}{2} \leq |\Pi_\boot| = {\cal O}(n^2)\, .
\end{align}
Thus,  the size of $\Pi_\boot$ is a bound to the number of tests. In general, because of the regularization term in the model fit of equation (\ref{eq:firstboot}), one expects   $ |\Pi_\boot|  \ll n^2$.


\section{Case studies} \label{sec:cases}

We performed extensive comparisons of our approach to the \removed{standard} \added{baseline} Hill Climbing by generating synthetic data.  Then, we tested the algorithm  against a well-known \BN{} benchmark, and against real cancer genomics data. We provide {\sc R} implementation of all the methods mentioned in this manuscript, as well as  sources   to replicate all our findings  (Supplementary Data). For Hill Climbing, we used 
the  {\sf bnlearn} package \cite{bnlearn}.

\subsection{Tests with synthetic data}

\begin{figure}[t]
\centerline{\includegraphics[width = .85\textwidth]{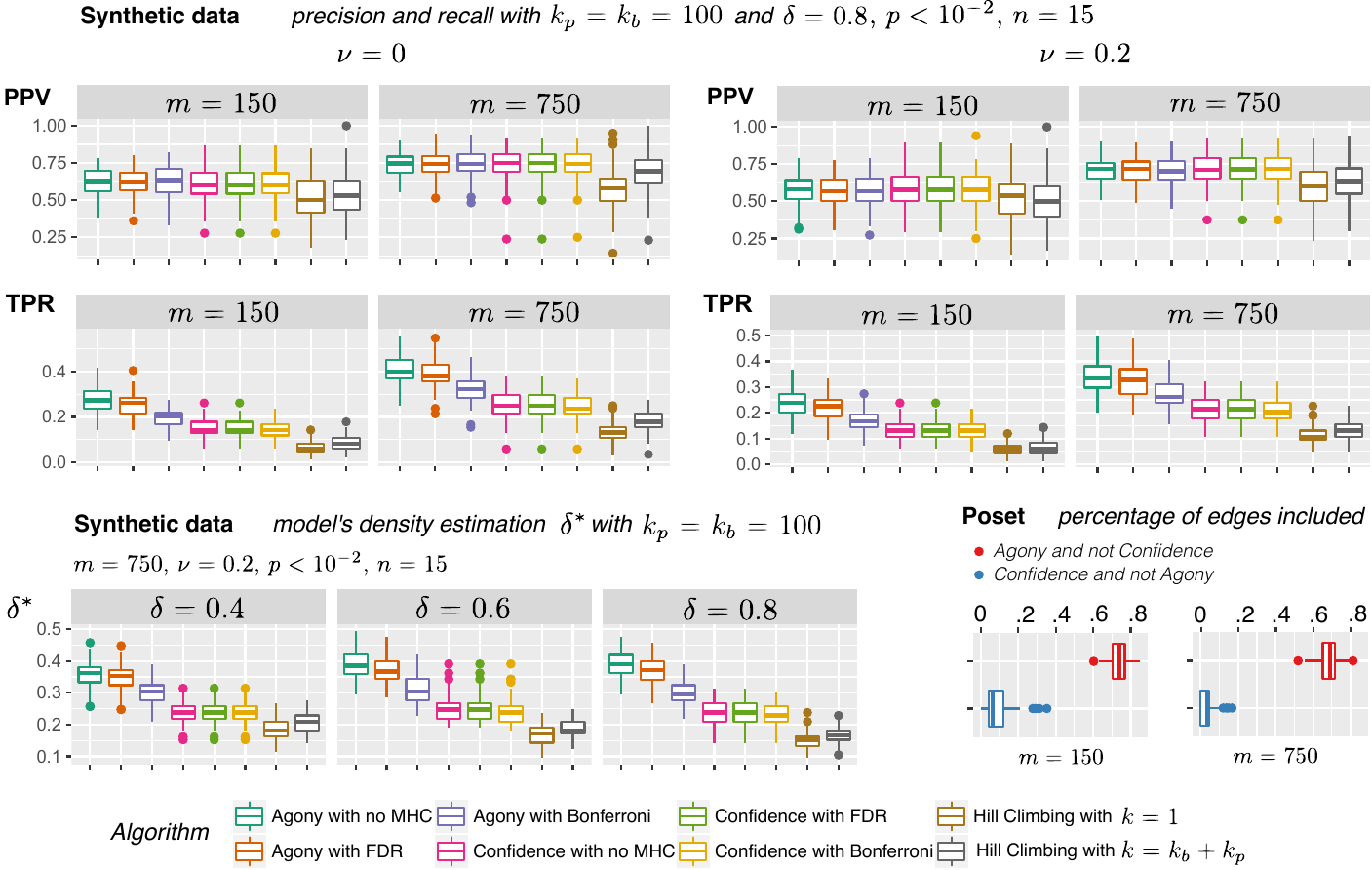}}
\caption{{Performance with synthetic data for binary variables with $\reg=\BIC$.} 
In top panel we show precision ($\PPV{}$) and recall 
 (\TPR{}) for \BN{}s with $n$ nodes,  density $\delta$, and $m$ samples perturbed at noise  rate $\nu$.  We compare  Hill Climbing   with $k=\{0,200\}$ ($\data \fit{\reg}{k}{\daguniv} \model$) against Algorithm \ref{alg1} with  $k_p=k_b=100$. $100$ \BN{}s for each parameter configuration are generated. The trends suggest a similar \PPV{} but better \TPR{} for Algorithm \ref{alg1} in all settings. The performance with the {\sf confidence} $\Pi$ seems independent of multiple hypotheses correction, which instead impacts on the performance with the {\sf agony} $\Pi$ (\FDR{} $0.2$).   Other tests carried out for $n=10$, $\delta= \{0.4,0.6\}$, $m=\{50, 100\}$, continuous variables \added{and Bayesian scores confirm these trends (Supplementary Figures \ref{fig:supp-test}, 
  \ref{fig:supp-test_dirk-BDE},   \ref{fig:supp-test_dirk-K2}, and \ref{fig:supp-test_dirk-BGE}}). In \added{the} bottom-left panel we show the density of the inferred models for different values of $\delta$, highlighting the intrinsic tendency of the plain regularization to low $\delta_\ast$. In the bottom-right panel we measure the overlap between the posets $\Pi$ built by {\sf confidence} or {\sf agony}, providing evidence to support Proposition \ref{ref:prop-posets}. 
%
}
\label{fig:synth}
\end{figure}

We carried out an extensive performance test that we recapitulate here and in the  Supplementary Material. \newSecond{A summary of all the considered configurations is provided in Supplementary Table \ref{table:summary_experiments}}. The aim of the test is to assess which configuration of poset and hypotheses testing performs best for Algorithm \ref{alg1}, and compare its performance against Hill Climbing.
We generated random networks (structures and parameteres) with different {\em densities} -- i.e., number of edges with respect to number of variables -- and various number of variables. From those \BN{}s \added{and a random (uniform) probability associated with each edge}, we generated several datasets and perturbed them with different rates of false positives and negatives ({\em noise}). For each model inferred, we computed standard scores of {\em precision} (positive predictive value, $\PPV{}$) and {\em recall} (true positives rate \TPR{}).

\added{Results for discrete networks with the $\reg=\BIC$ are shown in Figure \ref{fig:synth}. For continuous networks (Gaussian)  with also  $\reg=\AIC$ in  Supplementary Figure \ref{fig:supp-test}. Analogous tests for Bayesian scoring functions are in Supplementary Figures \ref{fig:supp-test_dirk-BDE} ($\reg=\BDE$),   \ref{fig:supp-test_dirk-K2} ($\reg=\KSCORE$) and \ref{fig:supp-test_dirk-BGE} ($\reg=\BGE$).}
 The comparison suggests that Algorithm \ref{alg1} has a similar ability to retrieve true edges of Hill Climbing, \PPV{}, but also a tendency to retrieve models with more edges, \TPR{}. Thus,  in all settings Algorithm \ref{alg1}  seems to improve remarkably  over the \removed{standard} \added{baseline} approach. The comparison suggests also that edge-selection by hypotheses testing seems less biased towards returning sparse models than a procedure based only on regularization.  However, both approaches seem to converge towards fixed densities of the inferred model, with Algorithm \ref{alg1} giving almost twice as many edges as Hill Climbing. 

The effect of $k$  independent initial conditions for the Hill Climbing procedure does not  seem to provide noteworthy improvements\footnote{Correlated restarts improve  Hill Climbing solutions (data not shown). However, for a fair comparison with Algorithm \ref{alg1} we should have then  correlated the initial solutions used to compute $\Pi$. 
To avoid including a further layer of complexity to all the  procedures, we rather not do that.}. Similarly,  strategies for \MHC{} do not seem to increase the performance in a particular way. For for {\sf agony}, a stringent correction -- \FWER{} -- seems too reduce \TPR{}, while \FDR{} does not seem to affect the scores. \MHC{} does not seem to have any effect on the {\sf confidence} poset. Interestingly, the comparison provides evidence that the {\sf agony} poset is a superset of the {\sf confidence} one, as the percentage of edges of the latter missing from the former approaches almost $0$. Other tests with these data suggest a minor improvement of performance if we use $1000$ bootstrap resamples, or different configurations of the parameters (data not shown). It is worth also to observe that, concerning the second bootstrap to create the null models, 
no major changes  where detected for larger $k_b$; so in practice $k_b=100$ could be considered as a suitable value across multiple application domains.

\subsection{The {\sf alarm} network}

\begin{figure}[t]
\centerline{\includegraphics[width = 1\textwidth]{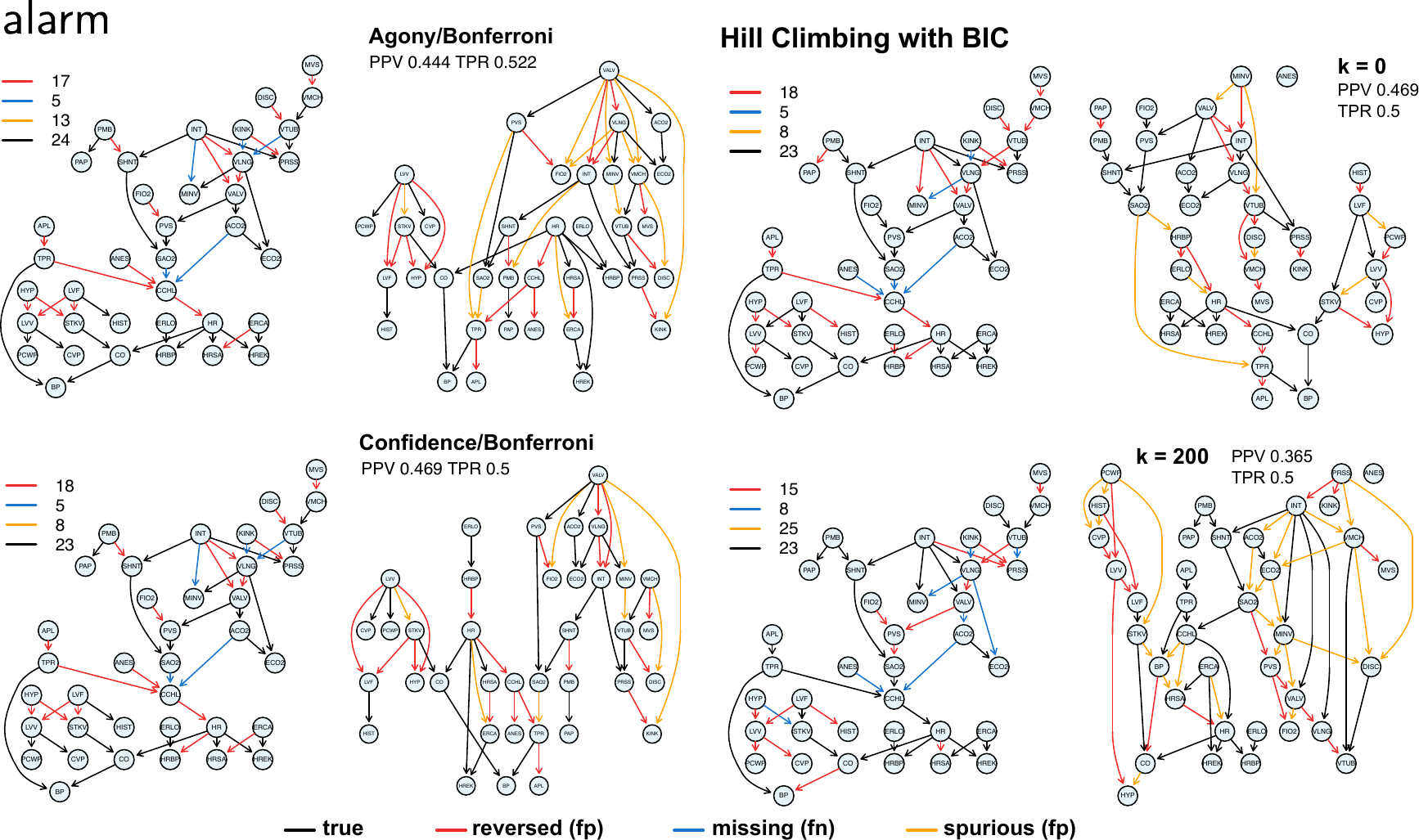}}
\caption{{Model selection for the {\sf alarm} network with  $m=10^5$ samples, and $\reg=\BIC$.} We compare  Hill Climbing   with $k=0$ and $k=200$  ($\data \fit{\reg}{k}{\daguniv} \model$) against Algorithm \ref{alg1} with  $k_p=k_b=100$. The left model of each pair is  {\sf alarm}, the right is $\model$. Edges are classified by color, depending which kind of false positive or negative they represent, and   precision and  recall scores are annotated. Algorithm \ref{alg1} ({\sf confidence}, \FWER) achieves the best scores with Hill Climbing   with $k=0$; for $k=200$ the Hill Climbing solution shows overfit. The models inferred by Algorithm \ref{alg1}  are strictly contained, and the confidence poset has higher scores than the agony poset. 
}
\label{fig:alarm}
\end{figure}

\begin{figure}[t]
\centerline{\includegraphics[width = 1\textwidth]{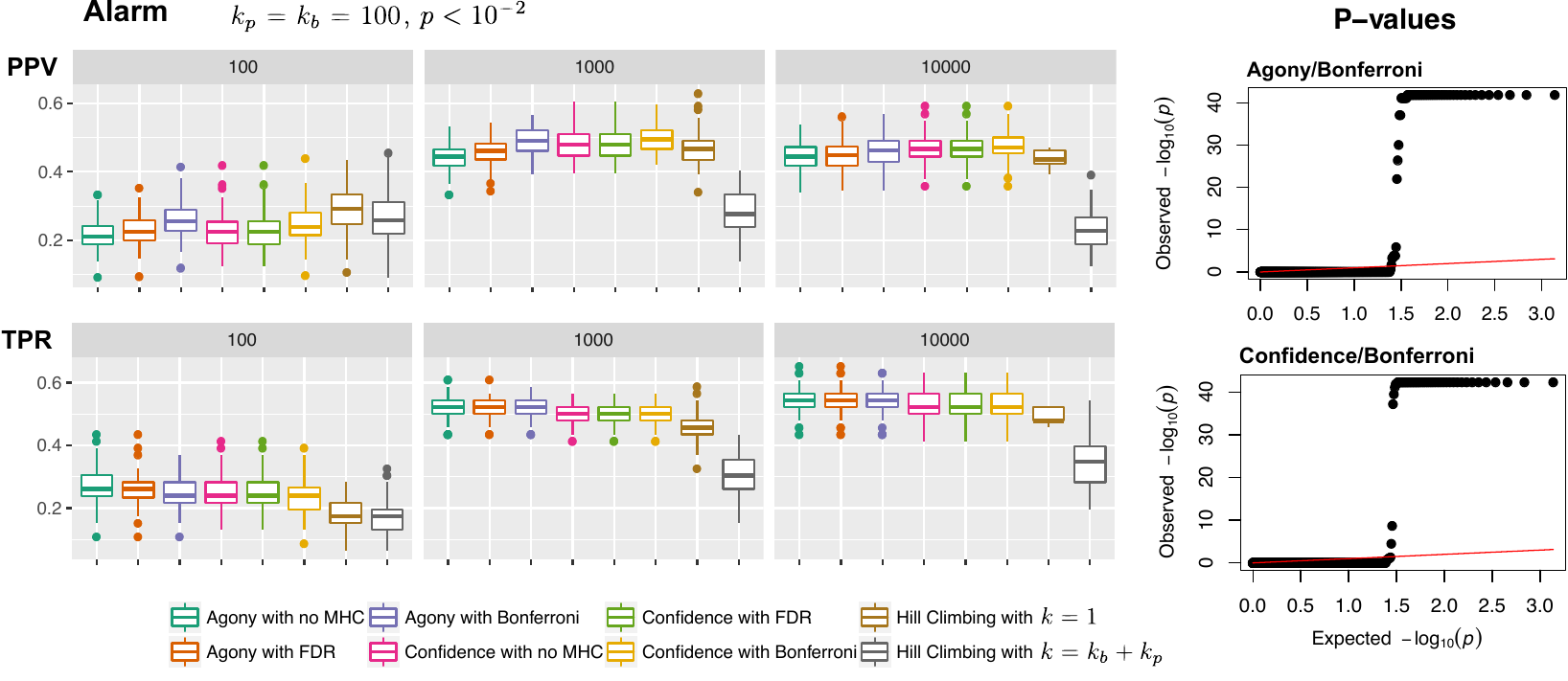}}
\caption{For different sample size $m$ we generated $100$ datasets to generalize the comparison of Figure \ref{fig:alarm}. The boxplots show the distributions of $\PPV{}$ and $\TPR{}$ for the {\sf alarm} network with  $m$ samples. The log-log plots show the gap of the p-value statistics for the two models computed by Algorithm \ref{alg1} and shown in Figure \ref{fig:alarm}. 
}
\label{fig:alarm_stats}
\end{figure}

We consider the standard    {\sf alarm} network \cite{alarm} benchmark, as provided in the {\sf bnlearn} package \cite{bnlearn}. {\sf alarm}   has $n=37$ variables connected through 46 edges, for a total of 509 parameters.

In Figure  \ref{fig:alarm} we show the result of  model selection for  large samples size and $\reg=\BIC$. The comparison is performed against  Hill Climbing   with $k=0$ and $k=200$, whereas  Algorithm \ref{alg1} is executed with  $k_p=k_b=100$. For large $m$, most setting seem to achieve the same performance; for lower $m$,  highest \PPV{} and \TPR{} are achieved by Algorithm \ref{alg1} ({\sf confidence}, \FWER{}). For this model, the use of multiple initial conditions for the  Hill Climbing   procedure  reduces \TPR{}; this is  due to the number of spurious edges estimated, as the number of true positives is the same for $k=0$ and $k=200$.  The models inferred by Algorithm \ref{alg1}  are strictly contained, and the confidence poset has higher scores than the agony one. 
 
 For this particular network we investigated also the effect of different sample size $m$, and the p-value for the statistical test on the performance of the algorithms. In Figure \ref{fig:alarm_stats} we show boxplots obtained from   $100$ datasets generated with different sample sizes. Results suggest minor changes in the performance with $m\geq10^3$, and generalize the findings of Figure \ref{fig:alarm}. Log-log plots  show a consistent gap in the  p-value statistics for the two models computed by Algorithm \ref{alg1}  shown in Figure \ref{fig:alarm}.  This is a phenomenon that we observed in all synthetic tests for sufficiently large $m$ (data not-shown), and that  suggests the correctness of the statistical test in Algorithm \ref{alg1}.
 
 Analysis of the variation of the performance as a function of the p-values' cutoff -- for $p<0.05$, $p<10^{-2}$ and  $p<10^{-3}$ with $m=100$ -- shows  small increase in \PPV{} for lower p-values, but not meaningful changes in \TPR{} scores (Supplementary Figure \ref{fig:alarm-pval}).

\new{As a final remark, we note that with this dataset  standard Hill Climbing without multiple restarts seems to achieve a better performance, compared to a search where multiple restarts are performed (see Figure \ref{fig:alarm_stats}). This behaviour might suggest the presence of a non-trivial relation underlying the ruggedness of the fitness landscape of the optimisation problem, and the role of restarts computed from correlated solutions. This kind of relation might require the development of more advanced resampling strategies, which could be approached leveraging on a bootstrap-based framework. 
}

\subsection{Modeling cancer evolution from genomic data}

\begin{figure}[p]
\centerline{\includegraphics[width = 1\textwidth]{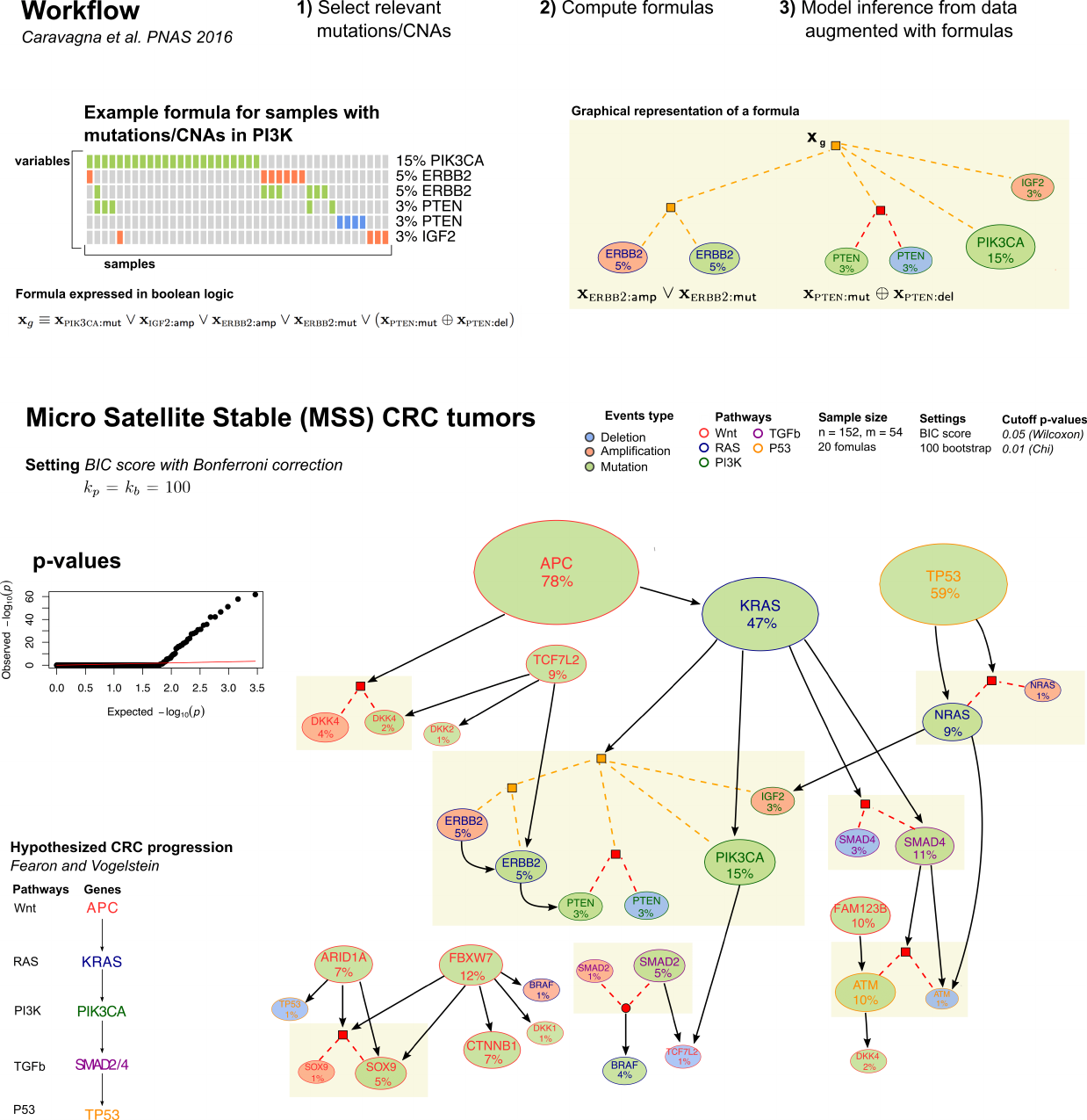}}
\caption{  We estimated a model of progression of colorectal cancer (CRC) from a set of MSS tumors  studied in \cite{Caravagna}.  Before  inference, a set of boolean formulae is computed and added to the input data as  new variables. These represent non-linear combinations of  mutations and copy numbers alterations (CNAs) in the original genes, as computed in \cite{Caravagna}. In top, we show the graphical notation of a formula that involves the genes activating the PI3K pathway; the intuition of  a formula is to capture a functional module that is disrupted by mutations/ CNAs differently across all patients. The model is then obtained  with $k_p=k_b=100$ and the same $\Pi_{\sf Suppes}$ estimated in \cite{Caravagna}  via Wilcoxon test ($p<0.05)$, after the marginal and conditional distributions are assessed with  $k_p$ bootstrap resamples.  In the test construction ($p<0.01$), we also use $100$ correlated restarts of the Hill Climbing to get better estimates for $\Gamma$. The linear progression model is due to Fearon and Vogelstein \cite{CRCprog}.
}
\label{fig:mss}
\end{figure}

\newcommand{\HL}[1]{{\textcolor{black}{\em #1}}}
\newcommand{\RT}[1]{}

Cancers progress by  \HL{accumulating genetic mutations} that allow cancer cells to grow and proliferate out of control \cite{Nowell}. Mutations  occur by chance, i.e., as a random process, and  are inherited through  divisions of cancer cells. The subset of  mutations  that trigger cancer growth by allowing a clone to expand, are called {\em drivers} \cite{Vogelstein}. Drivers, together with  epigenetic alterations,  orchestrate cancer initiation and development  with accumulation and activation patterns differing between {\em individuals} \cite{Michor}. This huge genotypic diversity -- termed  {\em tumor heterogeneity} --  is thought to lead to the emergence of drug-resistance mechanisms and failure of treatments \cite{Swanton}.  

Major efforts are ongoing to  decipher  the causes and consequences of tumor heterogenity, and its relation to tumor progression (see, e.g., \cite{TH}, and references therein).  Here, we   consider the problem of inferring a  {\em probabilistic model of cancer progression} that recapitulates the temporal ordering, i.e., {\em qualitative clocks}, of the mutations that accumulate during  cancer evolution \cite{Niko}. We do this by scanning   snapshots of cancer genomes collected   via biopsy samples of several primary tumors; all the patients are untreated  and diagnosed with the same cancer type (e.g., colorectal). 

In this model-selection problem variables are  $n$ somatic mutations detected by DNA sequencing   -- e.g., single-nucleotide mutations or chromosomal re-arrangements  -- annotated across $m$ independent  samples. Thus, a sample is an $n$-dimensional binary vector: $\mathbb{B}=\{0,1\}$,  and $\xx_i=0$ if the $i$-th lesion is not detected in  the   patient's cancer genome. We aim at inferring a  model that accounts for the accumulation of the input variables during tumor evolution in different patients.

\BN{}s do not encode explicitly this ``cumulative'' feature; however, they were recently  combined with Suppes'
 theory of {\em probabilistic causation}  \cite{Suppes}, which allows to describe cumulative phenomena.  {\em Suppes-Bayes Causal Networks} (\SBCN{}s, \cite{SBCN})  are \BN{}s whose  edges satisfy Suppes' axioms for {probabilistic causation}, which  mirror an expected ``trend of  selection'' among the lesions,  which is at the base of a Darwinian   interpretation of cancer evolution \cite{Nowell}.  Suppes' conditions take the form of inequalities over pairs of variables that are evaluated before  model-selection via a non-parametric bootstrap procedure. The  model-selection's landscape is then pruned of the edges that do not satisfy such conditions; thus, we can frame this as a poset 
 \begin{align}
\Pi_{\sf Suppes} = \{\xx_i \to \xx_j \mid p(\xx_i) > p(\xx_j) \wedge   p(\xx_j \mid \xx_i) > p(\xx_j\mid \neg \xx_i)  \}
\end{align}
that we  estimate from $\data$, along the lines of \cite{Ramazzotti}. The parameters $\bth$ of a \SBCN{} will encode these conditions  implicitly, rendering them suitable  to model cumulative diseases such as cancer or other diseases \cite{SBCN}.

We will use data from \cite{Caravagna}, which collected and pre-processed high-quality genomics profiles from The Cancer Genome Atlas\footnote{\href{https://cancergenome.nih.gov/}{https://cancergenome.nih.gov/}} (TCGA). We consider a dataset of $m=152$ samples and $n=54$ variables, which refers to {\em colorectal cancer patients} with clinical Microsatellite Stable Status\footnote{
The study in \cite{Caravagna} analyses also highly Microsatellite Instable tumors. Unfortunately, that subtype's data are associated  to a very small dataset of $m=27$ samples, and thus we here focus only on Microsatellite Stable tumors, a common subtype classification of such tumors.} (MSS). The  input data for MSS tumors consists in mutations ({\sf mut}, mostly missense etc.) and copy numbers ({\sf amp}, high-level amplifications; {\sf del}, homozygous deletions) detected in $21$ genes of 5 pathways that likely drive colon cancer progression \cite{TCGACRC}. 20 out of $54$ variables are obtained as non-linear combinations of  mutations and copy numbers in the original genes. For instance, 
\[
\xx_g \equiv \xx_{\text{\sc pik3ca}:{}\sf mut} \vee
\xx_{\text{\sc igf2}:{\sf amp}}\vee
\xx_{\text{\sc erbb2}:{\sf amp}}\vee
\xx_{\text{\sc erbb2}:{\sf mut}}\vee
(\xx_{\text{\sc pten}:{\sf mut}} \oplus
\xx_{\text{\sc pten}:{\sf del}})
\]
is a variable $\xx_g$ associated to the combination (in disjunctive $\vee$ and exclusive  $\oplus$ form) of the events associated to the driver genes of the {\sc pi3k} pathway  {\sc pik3ca}, {\sc igf2}, {\sc erbb2} and {\sc pten}.  These new  variables are called {\em formulas} (see \cite{Caravagna} for a full list) and are included in $\data$ before  assessment of Suppes' conditions for two reasons. They capture the {\em inter-patient heterogenity} observed across the TCGA cohort (i.e., as  biological ``priors''). They limit the confounding  effects of attempting inferences from hetergenous populations (i.e.,  as  statistical ``priors'').

We execute only the second part of our algorithm, i.e., the test, and compare the inferred model against the one obtained by  
Hill Climbing constrainted by  $\Pi_{\sf Suppes}$ and   with one initial condition (Figure 4 in \cite{Caravagna}).  In Figure \ref{fig:mss} we show the model obtained with $k_p=k_b=100$ and the same $\Pi_{\sf Suppes}$ estimated in \cite{Caravagna}  via Wilcoxon test ($p<0.05)$ after the marginal and conditional distributions are assessed with  $k_p$ bootstrap resamples. In the test construction ($p<0.01)$, we also use $100$ correlated restarts of the Hill Climbing to get better estimates for $\Gamma$.

We observe how our model is capable of capturing a lot of known features of MSS tumors as described in the seminal work of \cite{CRCprog}. In fact, we find APC as the main gene starting the progression followed by KRAS. Afterward, we observe multiple branches, yet involving genes from the PI3K (i.e., PIK3CA) and TGFb (i.e., SMAD2 and SMAD4) pathways, which are suggested to be later events during tumorigenesis of MSS tumors. While TP53 is not inferred to be a late event in the progression, we still find the P53 pathway to be involved in advanced tumors with ATM being one of the final nodes in one branch of the model. We remark that this tumor type shows considerable heterogeneity across different patients \cite{guinney2015consensus}, and evidences of TP53 as an early event in this cancer's progression have been found \cite{rivlin2011mutations}.

\section{Conclusions} \label{sec:conc}

In this paper we consider the identification of a factorization of a  \BN{} without hidden variables. This model-selection task is central to problems in statistics that require the learning of a  joint distribution made compact by retaining only the relevant conditional dependencies in the data. 

A common approach to it consists of a heuristic search over  the space of  factorizations,  the result being the  computation of the \MLE{} of the  structure {\em and}  the parameters  of the model, \added{or of a marginalised likelihood over the structures}. Surprisingly, the simple Hill-Climbing search strategy augmented with a regularized score function,  provides satisfactory baseline performance \cite{gamez2011learning}.

Here, we derive an algorithm based on  bootstrap and multiple hypothesis testing  that, compared to \removed{standard} \added{baseline} greedy optimization, achieves consistently better model estimates.  This result can stimulate further studies on the theoretical relation between the $\log$-likelihood function of a \BN{} and  greedy optimization, and attempts also at unifying two streams of research in \BN{} model-selection. 

On one side, we draw  inspiration from the seminal works by Friedman {\em et al.} which investigated whether we can assess {\em ``if   the existence of an edge between two nodes is warranted''}, or if  we {\em ``can  say something of the ordering of two variables''} \cite{Friedman}. Precisely, Friedman  {\em et al.} answered to these questions by showing that high-confidence estimates on certain structural features, when assessed by  a non-parametric bootstrap strategy,  can be {\em ``indicative of the existence of these features in the generative model''}. 

On the other side, we follow the  suggestion by Teyssier and Koeller on  the well-known fact that the best network consistent with a given node ordering can be found very efficiently \cite{teyssier2012ordering}.  Teyssier and Koeller  consider \BN{}s of bounded in-degree, and {\em ``propose a search not over the space of structures, but over the space of orderings, selecting for each ordering the best network consistent with it''}. Their motivation is driven by algorithmic an argument: {\em ``[the orderings'] search space is much smaller, makes more global search steps, has a lower branching factor, and avoids costly acyclicity checks''}.

Here, we connect the two observations in one framework. We first estimate orderings via  non-parametric bootstrap, combined with greedy estimation of the \removed{\MLE{}} model in each resample. Then, after rendering the model acyclic, we use it to select one final model that is consistent with the orderings. \added{Our approach  improves regardless of the information-theoretic or Bayesian scoring function adopted}.
To this extent, we use the orderings as an empirical Bayes prior over model structures, and compute the {\em maximum a posteriori} estimate of the model. The parameters are then the \MLE{} estimates for the selected structure. Our result is based on a refinement of the original observation by Teyssier and Koller: when we know the ordering, besides improving complexity we enjoy a systematic reduction in the ``statistical'' complexity in the problem of identifying true dependencies. We postulate this after observing that with the best possible ordering -- i.e., a transitive closure of the generative model -- the fitness landscape becomes unimodal.

\paragraph{Acknowledgement.}

\added{Both the authors wish to thank Guido Sanguinetti and Dirk Husmeier for useful discussions on a preliminary version of this manuscript.}

\bibliographystyle{unsrt}
\bibliography{biblio}

%
%
%


\FloatBarrier
\appendix

\renewcommand\thefigure{S\arabic{figure}}    
\renewcommand\thetable{S\arabic{table}}


\newpage

\newSecond{\section{Supplementary Tables}}

\newSecond{The following tables are provided.}

\begin{table}[H]
\centering
\begin{tabular}{cccccc}
\textbf{$\#$Simulations} & \textbf{Variables} & \textbf{$\#$Node} & \textbf{Sample Size} & \textbf{Density} & \textbf{Noise Level} \\
$100$ & Binary & $15$ & $\{75, 150, 750\}$ & $\{0.4, 0.6, 0.8\}$& $\{0.0, 0.2\}$ \\
$100$ & Binary & $10$ & $\{50, 100\}$ & $\{0.4, 0.6\}$& $\{0.0, 0.2\}$ \\
$100$ & Continuous & $10$ & $\{50, 100\}$ & $\{0.4, 0.6\}$& $\{0.0, 0.2\}$  
\end{tabular}
\caption{\newSecond{Performed synthetic tests. In this table we summarize the performed simulations. Namely, we considered $3$ settings; the first two for Bayesian Networks respectively of $10$ and $15$ binary variables, different sample sizes and network densities. In the third experiment, we considered a similar configuration to experiment $2$ but we now simulate continuous variables. For all datasets, we considered both the noise free case and the one with $20\%$ noise. We performed 100 independent simulations for each configuration. This led us to a total of $3400$ synthetic datasets.}}
\label{table:summary_experiments}
\end{table}

\newpage

\section{Supplementary Figures}

The following figures are provided.

\begin{itemize}
\item Figure \ref{fig:supp-test}: {\em synthetic tests with different settings from Figure \ref{fig:synth}.}
\item \added{Figure \ref{fig:supp-test_dirk-BDE},  \ref{fig:supp-test_dirk-K2} and  \ref{fig:supp-test_dirk-BGE}: {\em synthetic tests analogous to the ones from Figure \ref{fig:synth} for Bayesian scoring functions.}}
\item Figure \ref{fig:alarm-pval}: {\em the effects of different p-values on the model-selection for the {\sf alarm} network.}
\end{itemize}

\begin{figure}[t]
\includegraphics[width = .8\textwidth]{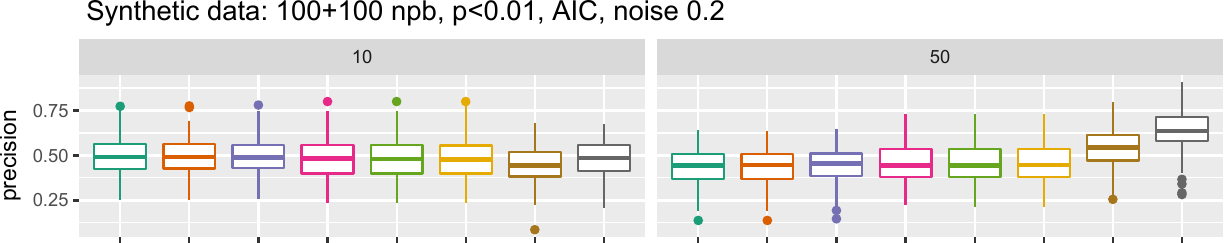}
\includegraphics[width = .8\textwidth]{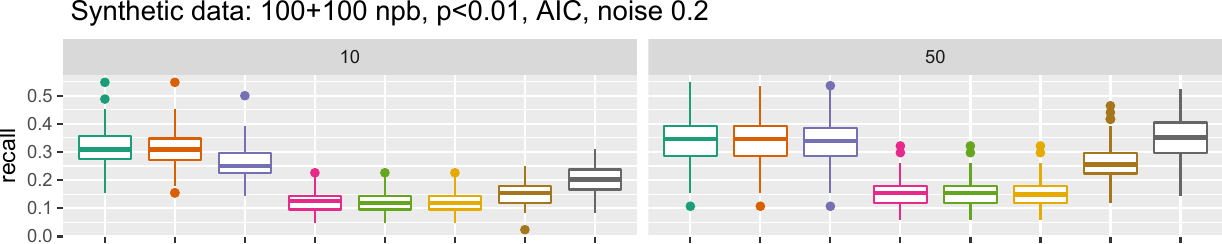}
\includegraphics[width = .8\textwidth]{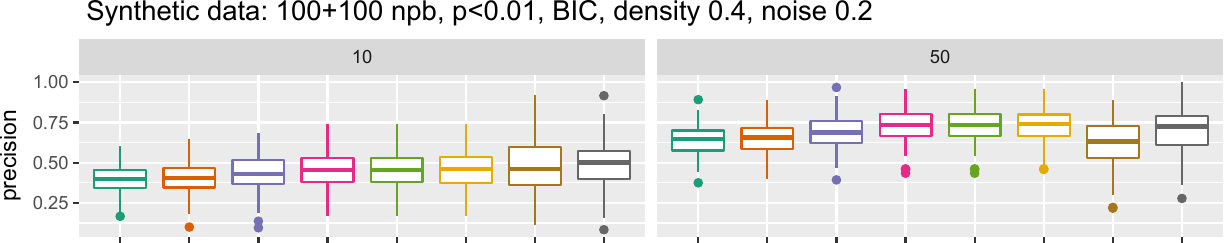}
\includegraphics[width = .8\textwidth]{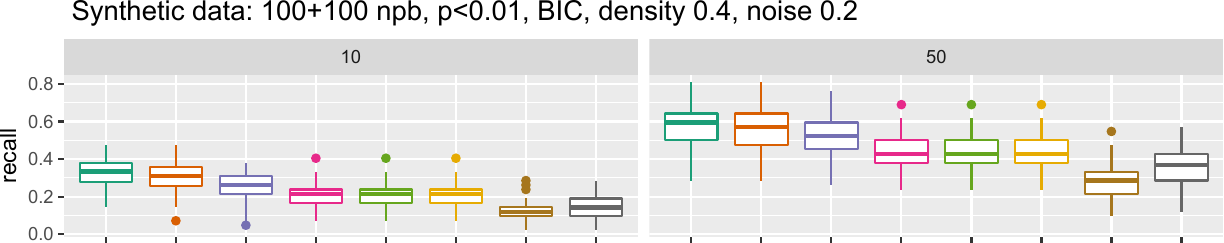}
\includegraphics[width = .8\textwidth]{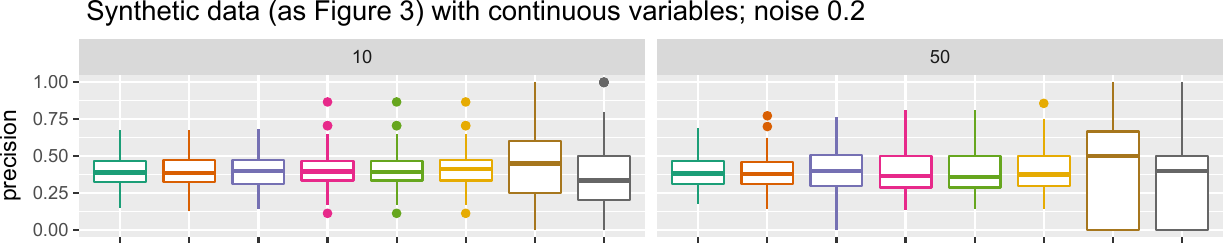}
\includegraphics[width = .8\textwidth]{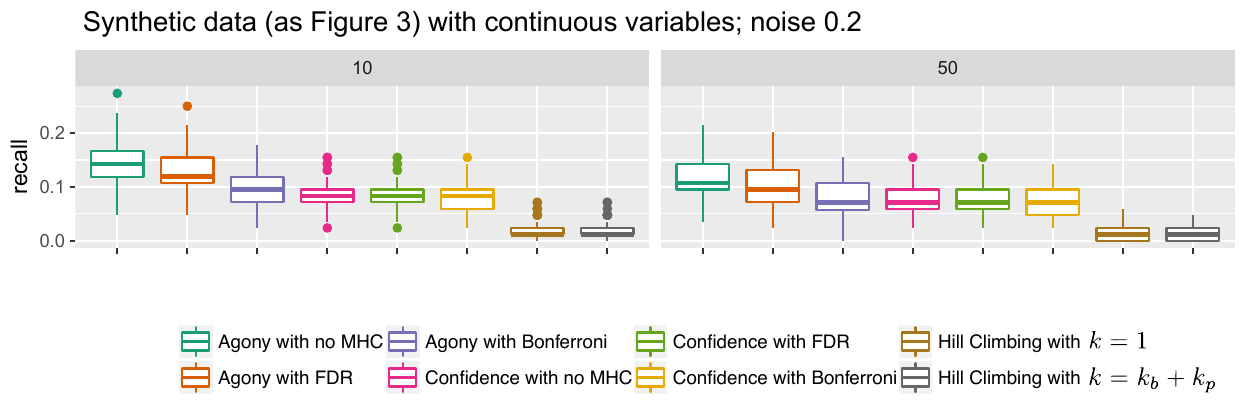}
\caption{Synthetic tests with different settings from Figure \ref{fig:synth}: top, \reg=\AIC, mid, $\delta = 0.4$, and bottom, continuous variables. In left, for $n$ the number of nodes in the model, we generate $10*n$ samples, in right $50*n$. 
}
\label{fig:supp-test}
\end{figure}


\begin{figure}[t]
\centerline{
\includegraphics[width = .8\textwidth]{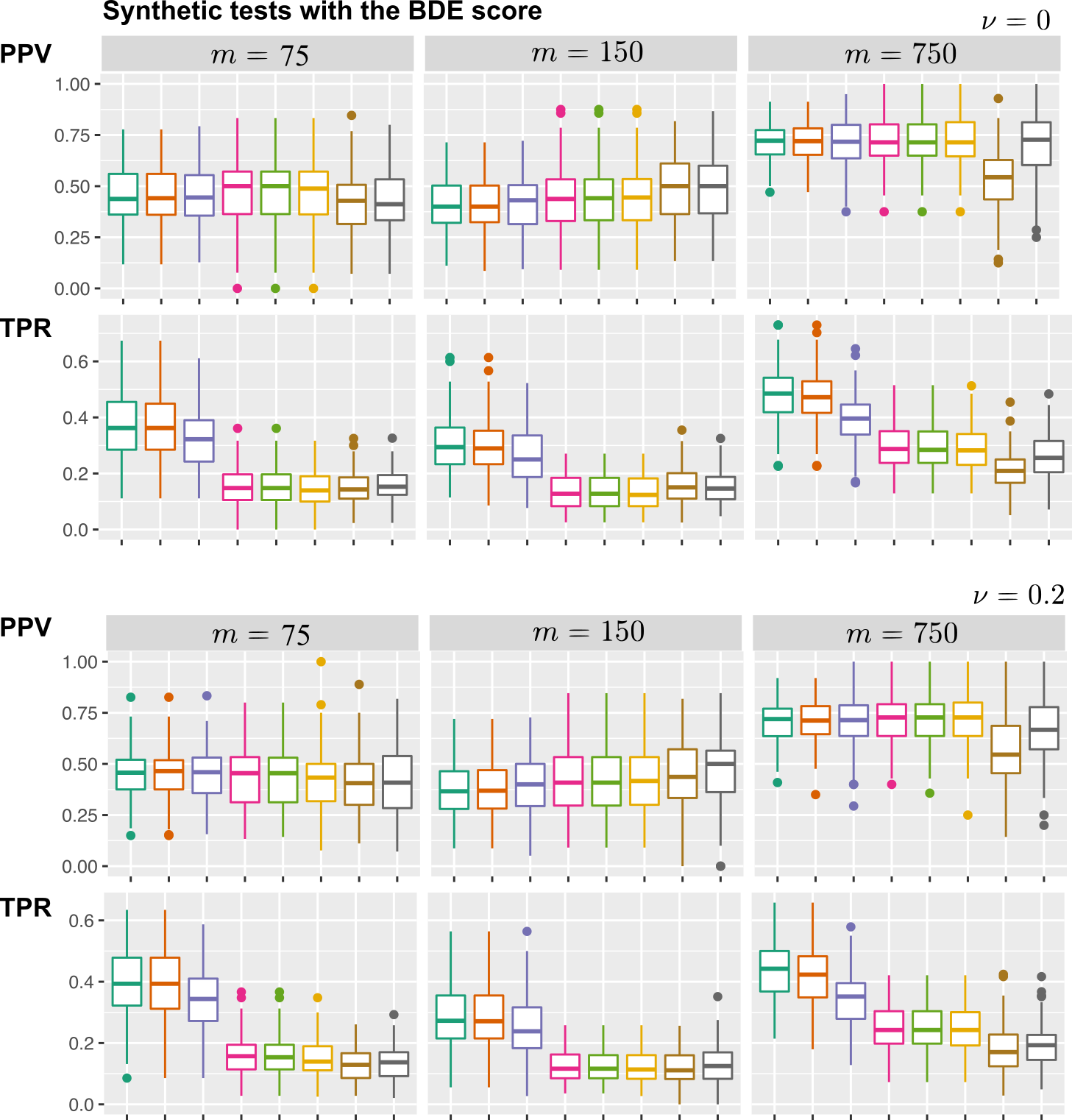}
}
\caption{\added{Synthetic tests  with binary variables  for the \BDE{} Bayesian score. We observe that these simulations, as well as those for 
other Bayesian scores (Supplementary Figures \ref{fig:supp-test_dirk-K2} and \ref{fig:supp-test_dirk-BGE}) show similar trends to the ones discussed in the main text for information-theoretic scoring functions.} 
}
\label{fig:supp-test_dirk-BDE}
\end{figure}

\begin{figure}[t]
\centerline{
\includegraphics[width = .8\textwidth]{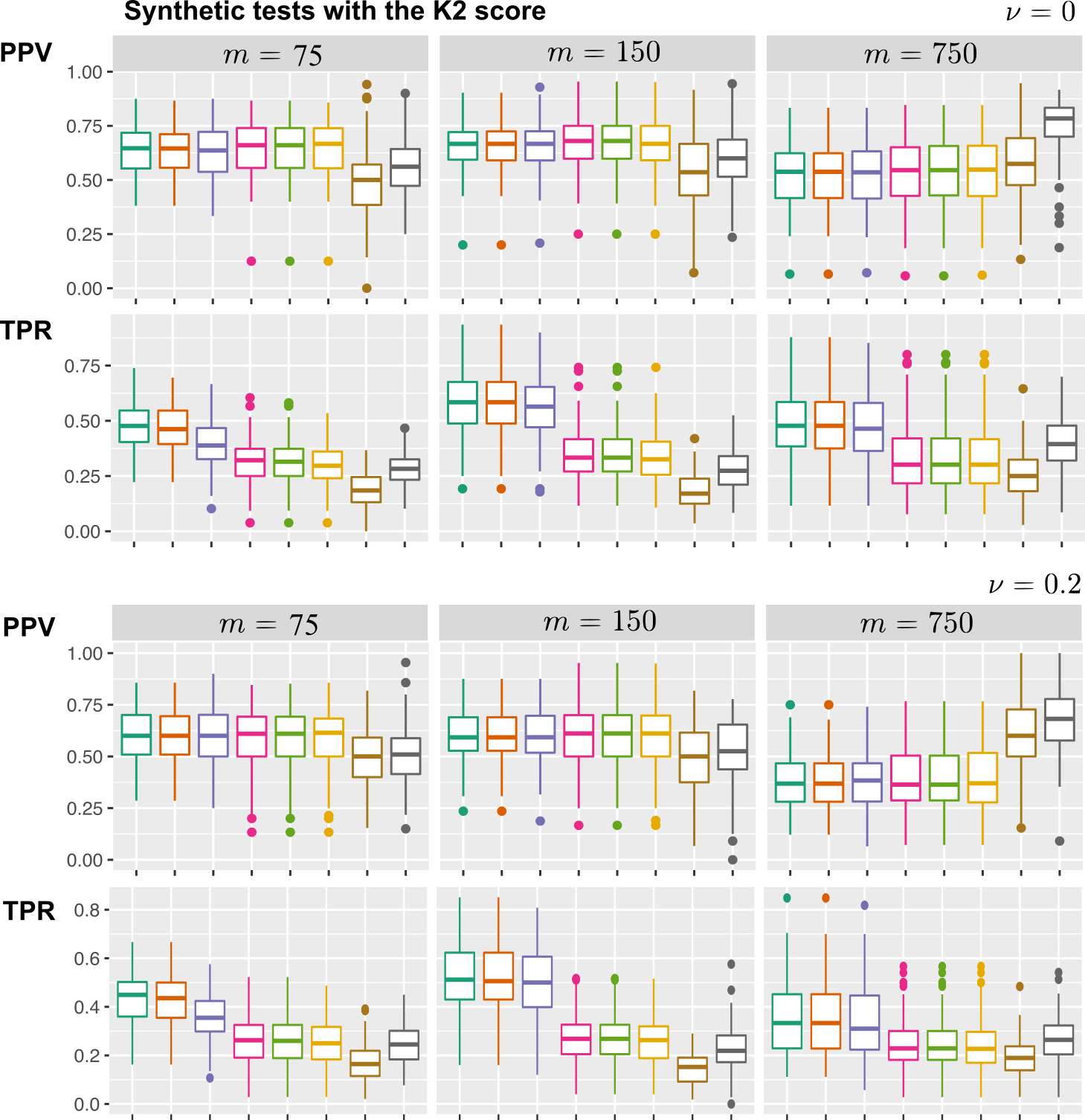}
}
\caption{\added{Synthetic tests  with binary variables  for the \KSCORE{} Bayesian score.} 
}
\label{fig:supp-test_dirk-K2}
\end{figure}

\begin{figure}[t]
\centerline{
\includegraphics[width = .8\textwidth]{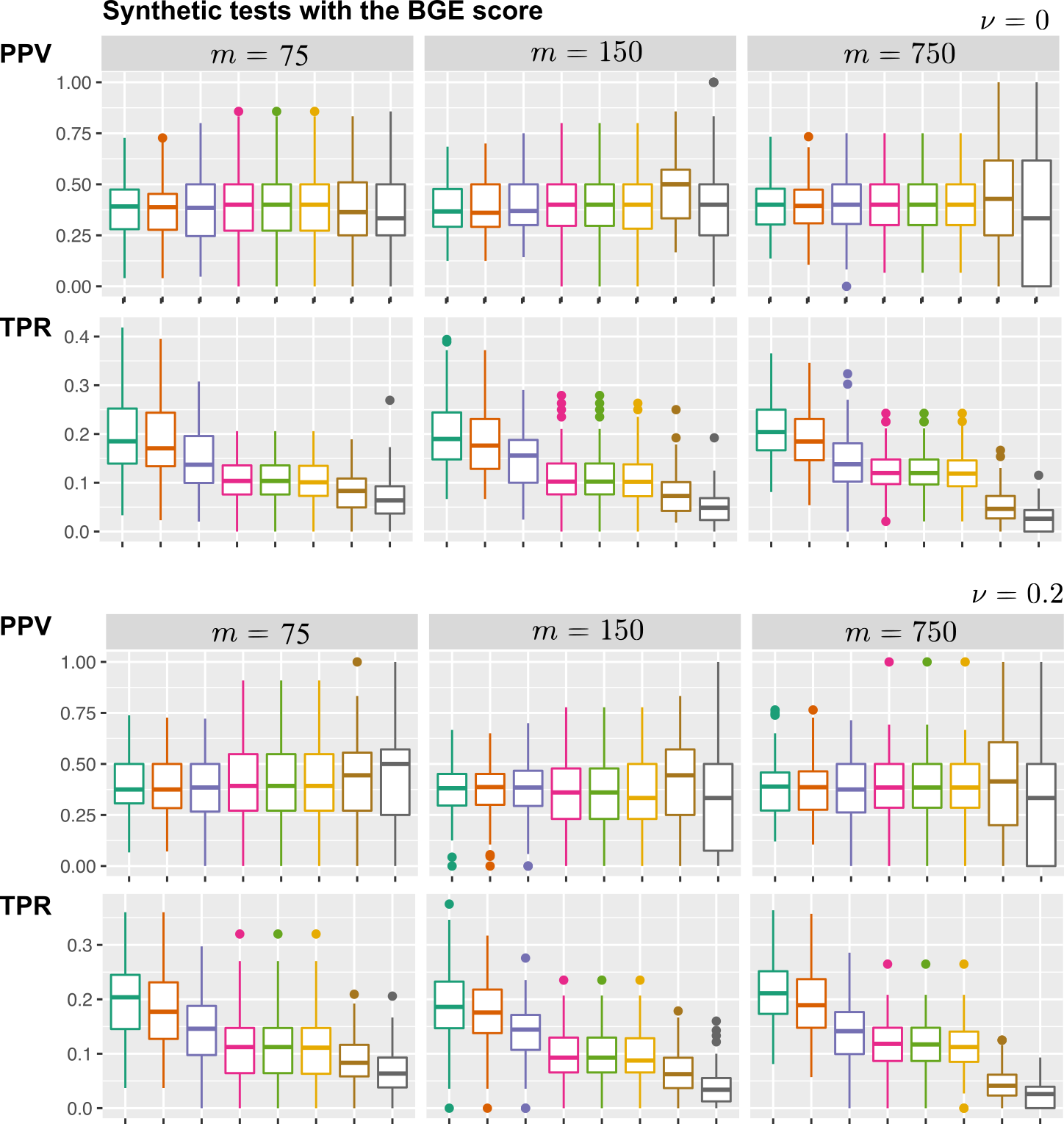}
}
\caption{\added{Synthetic tests  with Gaussian variables  for the \BGE{} Bayesian score. } 
}
\label{fig:supp-test_dirk-BGE}
\end{figure}


\begin{figure}[t]
\centerline{
\includegraphics[width = .8\textwidth]{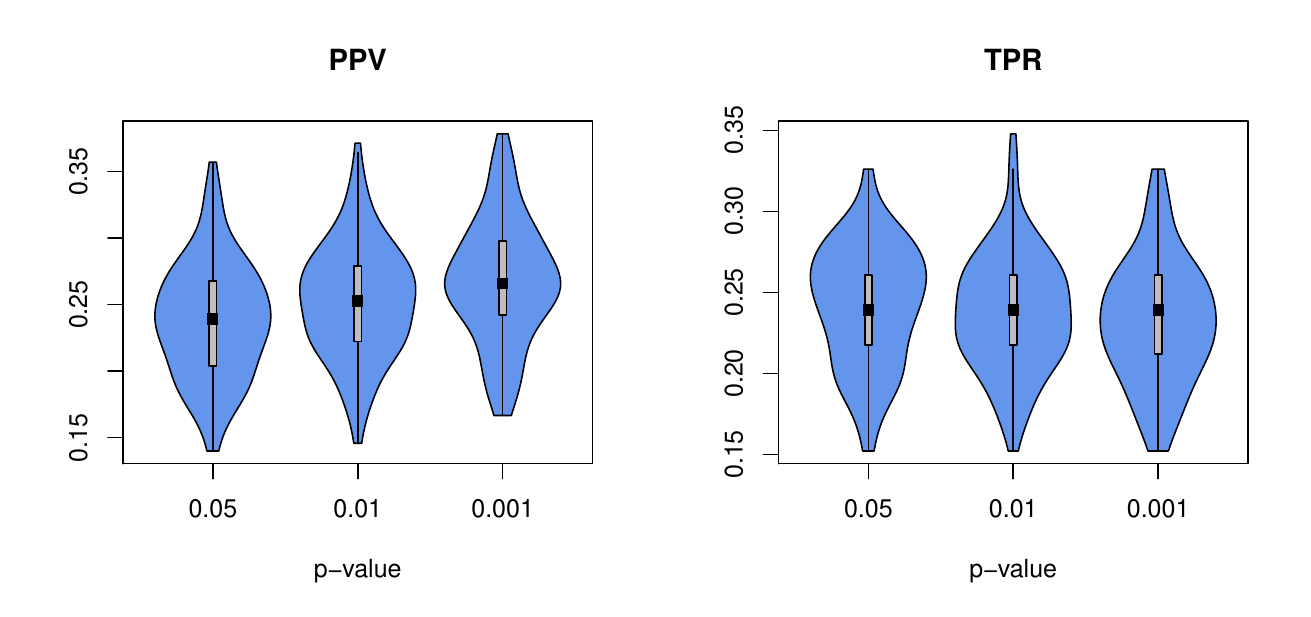}}
\caption{Violin plots for  different p-values $p$ on the model-selection for the {\sf alarm} network with the agony poset and Bonferroni correction. $100$ random datasets are generated with $m=100$ samples. The same settings of Figure \ref{fig:alarm_stats} are used.  
}
\label{fig:alarm-pval}
\end{figure}

\end{document}